\documentclass[lettersize, journal]{IEEEtran}

\usepackage[T1]{fontenc}      
\usepackage[utf8]{inputenc}   
\usepackage{textcomp}          
\usepackage{url}               
\usepackage{verbatim}          

\usepackage{amsmath}           
\usepackage{amsfonts}          
\usepackage{amssymb}           
\usepackage{mathrsfs}          

\usepackage{graphicx}          
\usepackage[caption=false,labelfont=sf,textfont=sf]{subfig}

\usepackage{booktabs}          
\usepackage{multirow}          
\usepackage{makecell}          
\usepackage{array}             
\usepackage{diagbox}           
\usepackage{tablefootnote}     
\usepackage{threeparttable}    

\usepackage{algorithm}         
\usepackage{algorithmic}       

\usepackage{balance}           
\usepackage{enumitem}          
\usepackage{color}             
\usepackage{xcolor}            
\usepackage{xspace}            
\usepackage{adjustbox}         
\usepackage{pifont}            
\usepackage{tikz}              
\usepackage{cite}              

\def\ie{\textit{i.e.}\xspace}

\def\eg{\textit{e.g.}\xspace}

\def\method{CoASFL\xspace}

\newtheorem{assumption}{\textbf{Assumption}}
\newtheorem{lemma}{\textbf{Lemma}}
\newtheorem{theorem}{\textbf{Theorem}}

\newcommand*{\circled}[1]{\lower.8ex\hbox{\tikz\draw (0pt, 0pt)%
    circle (.47em) node {\makebox[0.4em][c]{\small #1}};}}

\begin{document}

\title{CoASFL: Coupled Consistency Optimization for Asynchronous Split Federated Learning}


\author{
    Junhao Cheng,~
    Yunming Liao~,
    Hongli Xu,~\IEEEmembership{Member,~IEEE,}~
    Yang Xu,~\IEEEmembership{Member,~IEEE,}~
    Zuan Xie,~
    \IEEEcompsocitemizethanks{
        \IEEEcompsocthanksitem 
J. Cheng, Y. Liao, H. Xu, Y. Xu, and Z. Xie are with the School of Computer Science and Technology, University of Science and Technology of China, Hefei, Anhui, China, 230027, and also with Suzhou Institute for Advanced Research, University of Science and Technology of China, Suzhou, Jiangsu, China, 215123. \protect E-mails: jhcheng01@mail.ustc.edu.cn, ymliao98@mail.ustc.edu.cn, xuhongli@ustc.edu.cn, xuyangcs@ustc.edu.cn, xz1314@mail.ustc.edu.cn.
    }
}

\markboth{IEEE Transactions on Mobile Computing}%
{Cheng \MakeLowercase{\textit{et al.}}: CoASFL}


\maketitle

\begin{abstract}
Split Federated Learning (SFL) has emerged as a promising paradigm for training large-scale deep models on resource-constrained edge devices.
Although asynchronous SFL alleviates the straggler bottleneck caused by system heterogeneity, it inevitably introduces two coupled dimensions of inconsistency: horizontal inconsistency across clients and vertical inconsistency between clients and the server.
These inconsistencies, further exacerbated by statistical heterogeneity, result in biased activation representations, model divergence, and severe client drift.
To address these challenges, we propose \method, an inconsistency-aware asynchronous SFL framework that incorporates a dynamic activation merging mechanism and a consistency-aware model aggregation mechanism.
Specifically, \method constructs more consistent activation sequences from heterogeneous clients to alleviate horizontal inconsistency, and dynamically adjusts aggregation weights according to client participation frequencies and progress gaps to mitigate vertical inconsistency.
By jointly coordinating activation-level and model-level consistency, \method effectively suppresses inconsistency propagation during asynchronous training.
Extensive experiments on an 80-device heterogeneous edge testbed demonstrate that \method consistently outperforms state-of-the-art SFL approaches, accelerating training by 1.61$\times$ to 3.74$\times$ and improving final model accuracy by up to 3.26\% under heterogeneous scenarios.

\end{abstract}

\begin{IEEEkeywords}
Split Federated Learning, Asynchronous Training, Edge Intelligence, Consistency Optimization.
\end{IEEEkeywords}

\section{Introduction}\label{sec:intro}

\IEEEPARstart{T}{he} rapid advancement of deep learning has driven the deployment of intelligent applications at the network edge, including autonomous systems, intelligent healthcare, and industrial Internet of Things (IoT) \cite{li2019edge, wang2019edge}.
However, training increasingly large deep models on distributed edge devices remains challenging due to the limited computation capability, storage capacity, and communication resources of individual devices \cite{park2021communication, niu2020billion}.
Federated Learning (FL) provides a privacy-preserving paradigm for collaborative model training by keeping raw data locally while aggregating knowledge from distributed clients \cite{mcmahan2017communication, kairouz2021advances, liao2023accelerating}.
Nevertheless, conventional FL requires frequent exchanges of complete model parameters between clients and the central server, resulting in substantial communication and computation overhead when training large-scale models \cite{abedi2024fedsl}.
Split Federated Learning (SFL) addresses this limitation by partitioning the model into client-side bottom models and a server-side top model, where clients exchange intermediate activations and split-layer gradients rather than full model parameters \cite{pathak2020fedsplit, thapa2022splitfed}.
By reducing the resource burden on edge devices, SFL has become an attractive solution for collaborative intelligence in heterogeneous edge environments \cite{10910050, 11219319}.

Despite its advantages, conventional SFL generally follows a synchronous training paradigm, where the server must wait for all clients to complete local training before updating the global model \cite{liao2024parallelsfl}.
Such synchronization barriers severely restrict training efficiency in practical edge systems, where clients often exhibit significant heterogeneity in computation capability, network conditions, and resource availability \cite{li2023fedvs, shen2024ringsfl, gupta2023fedcare, liang2026towards}.
For instance, heterogeneous edge devices may require substantially different execution times to process identical workloads, while unstable wireless networks can further increase activation transmission delays \cite{xu2022adaptive, liao2023adaptive, liao2024asynchronous}.
Consequently, slow clients become stragglers that postpone global updates and leave the resources of faster clients underutilized.
To address this issue, recent studies \cite{gao2024pipesfl, yang2025gas, ao2025semi} have introduced asynchronous execution into SFL, allowing the server to continuously process available client updates without waiting for full synchronization.
Although asynchronous SFL effectively improves resource utilization and alleviates the straggler problem, it also fundamentally changes the optimization dynamics of SFL by introducing new consistency challenges.

Specifically, asynchronous SFL suffers from two coupled inconsistency dimensions: horizontal inconsistency across clients and vertical inconsistency between clients and the server.
Horizontal inconsistency arises because heterogeneous computation and communication capabilities cause activations to arrive at different rates. 
As the server updates the top model based on these arrivals, faster clients exert disproportionate influence on the optimization, and under non‑IID data this imbalance is further amplified, deviating the top model from the global objective and impairing convergence \cite{cho2022towards, ma2021fedsa}.
Vertical inconsistency arises because the server continuously updates the top model, while each bottom model is updated only through its own client interactions.
The resulting progress gaps distort split-layer gradients and bottom model updates, and their accumulation further biases bottom model aggregation and impairs convergence \cite{xie2019asynchronous, chen2020asynchronous}.
More importantly, these two inconsistencies are not independent and may interact during asynchronous training: inconsistent activations impair the quality of server-side optimization, whereas biased model aggregation may further amplify client drift among heterogeneous clients.
Therefore, jointly mitigating horizontal and vertical inconsistencies becomes a fundamental challenge for achieving efficient and accurate asynchronous SFL.

Existing asynchronous SFL approaches only partially address these challenges.
Some studies primarily focus on improving execution efficiency through asynchronous parallelism.
For example, PipeSFL \cite{gao2024pipesfl} introduces a fine-grained pipeline execution mechanism to overlap client computation and server processing, thereby reducing idle time caused by stragglers.
However, it does not consider the inconsistency among asynchronously received intermediate activations.
GAS \cite{yang2025gas} further investigates asynchronous SFL under heterogeneous data distributions by generating intermediate activations to alleviate the impact of statistical heterogeneity, but its aggregation strategy cannot effectively handle model divergence caused by stale client updates.
Meanwhile, MergeSFL \cite{10597905} improves robustness against heterogeneous data by merging intermediate activations from multiple clients, but it is designed for synchronous SFL and cannot fully exploit asynchronous execution.
Therefore, existing methods either improve system efficiency or alleviate heterogeneity issues, while lacking a unified framework that simultaneously addresses the coupled horizontal and vertical inconsistencies in asynchronous SFL.

To address the above challenges, we propose \method, an inconsistency-aware asynchronous SFL framework that jointly suppresses inconsistency propagation during asynchronous training.
Instead of treating horizontal and vertical inconsistencies separately, \method coordinates them through a unified consistency optimization framework.
Specifically, \method integrates dynamic activation merging and consistency-aware model aggregation into a coupled optimization strategy.
The dynamic activation merging coordinates activation merging at the server to improve the consistency of top model updates, while the consistency-aware model aggregation regulates bottom model aggregation under unequal client participation and progress mismatch.
Rather than independently optimizing activation representations or aggregation reliability, the two mechanisms reinforce each other to preserve activation-level and model-level consistency throughout training.
As a result, \method achieves both efficient convergence and high model accuracy under heterogeneous edge environments.
The main contributions of this paper are summarized as follows:

\begin{itemize}
    \item We identify the dual inconsistency problem in asynchronous SFL, including horizontal inconsistency across clients and vertical inconsistency between clients and the server, and analyze their impact on model convergence and accuracy.
    \item We propose \method, an inconsistency-aware asynchronous SFL framework that jointly coordinates activation-level and model-level consistency through dynamic activation merging and consistency-aware model aggregation, effectively suppressing inconsistency propagation throughout asynchronous training.
    \item We implement \method on an 80-device heterogeneous edge testbed and conduct extensive experiments. The results demonstrate that \method accelerates training by 1.61$\times$ to 3.74$\times$ and improves final model accuracy by up to 3.26\% under heterogeneous scenarios, compared to state-of-the-art SFL approaches.
\end{itemize} 

The rest of the paper is organized as follows.
Section \ref{sec:prelim} mainly introduces the impact of horizontal and vertical inconsistency.
Section \ref{sec:design} elaborates on the design of \method.
Section \ref{sec:conv} presents a convergence analysis of \method.
Section \ref{sec:evaluation} discusses the experimental results.
Section \ref{sec:relworks} introduces the related works.
Finally, Section \ref{sec:conclusion} concludes the paper.

\section{Background and Motivation}\label{sec:prelim}
\subsection{Split Federated Learning}
Split federated learning (SFL) enhances FL by integrating split learning \cite{thapa2022splitfed}, providing an effective solution for training large-scale models on resource-constrained devices.
We consider a typical SFL system that consists of a parameter server (PS) and $N$ clients.
\textit{By default, all clients use the same fixed batch size for training}.
The core concept of SFL is to introduce a split layer that divides the model $\boldsymbol{w}$ into a bottom model $\boldsymbol{w}_b$ retained by the clients and a top model $\boldsymbol{w}_p$ held by the PS, \ie, $\boldsymbol{w}=[\boldsymbol{w}_b,\boldsymbol{w}_p]$.
The basic training process of the synchronous SFL requires cooperation between the clients and the PS.

Let $\boldsymbol{w}^{h,k}_{b,i}$ denote the bottom model of client $i$ and $\boldsymbol{w}_{p}^{h,k}$ denote the top model at iteration $k$ in round $h$. 
Client $i$ performs forward propagation on $\boldsymbol{w}^{h,k}_{b,i}$ using a local batch sampled from its dataset $d_i$, and then sends the resulting activations to the PS.
The empirical loss associated with the bottom model of client $i$ is represented as:
\begin{equation}
F_{b,i}(\boldsymbol{w}_{b,i}^{h,k})
=
\frac{1}{|d_i|}
\sum_{x\in d_i}
\ell(x;
\boldsymbol{w}_{b,i}^{h,k},
\boldsymbol{w}_{p}^{h,k})
\label{eq:sfl_loss_sample}
\end{equation}

\noindent where $\ell(x;\boldsymbol{w}_{b,i}^{h,k},\boldsymbol{w}_{p}^{h,k})$ denotes the loss function computed by the current bottom and top models using sample $x$.
After receiving activations from all clients, the PS completes forward and backward propagation and updates the top model as:
\begin{equation}
\boldsymbol{w}_{p}^{h,k+1}
=
\boldsymbol{w}_{p}^{h,k}
-
\eta\cdot\frac{1}{N}
\sum_{i=1}^{N}
\widetilde{\nabla}
F_{p,i}(\boldsymbol{w}_{p}^{h,k})
\label{eq:sfl_top_update}
\end{equation}

\noindent where $F_{p,i}(\boldsymbol{w}_{p}^{h,k})$ denotes the loss computed using the activations uploaded by client $i$, $\widetilde{\nabla} F_{p,i}(\boldsymbol{w}_{p}^{h,k})$ is its stochastic gradient with respect to the top model, and $\eta$ is the learning rate.
The PS then returns the corresponding split layer gradient to client $i$, which updates its bottom model as:
\begin{equation}
\boldsymbol{w}_{b,i}^{h,k+1}
=
\boldsymbol{w}_{b,i}^{h,k}
-
\eta\cdot
\widetilde{\nabla}
F_{b,i}(\boldsymbol{w}_{b,i}^{h,k})
\label{eq:sfl_bottom_update}
\end{equation}
\noindent where $\widetilde{\nabla}F_{b,i}(\boldsymbol{w}_{b,i}^{h,k})$ is the stochastic gradient computed using the split layer gradient returned by the PS.
At the end of round $h$, the PS aggregates bottom models from all clients as: 
\begin{equation}
\boldsymbol{w}_{b}^{h+1}
=
\sum_{i=1}^{N}
A_i^h\boldsymbol{w}_{b,i}^{h},
\qquad
A_i^h
=
\frac{n_i^h}
{\sum_{j=1}^{N}n_j^h}
\label{eq:conventional_aggregation}
\end{equation}

\noindent where $n_i^h$ denotes the number of samples processed by client $i$ in round $h$, $A_i^h$ denotes its aggregation weight, and the aggregation reduces to uniform averaging when all $n_i^h$ are equal, \ie, $A_i^h=1/N$.

Unlike the synchronous SFL, asynchronous SFL removes the per-iteration barrier.
The PS immediately processes each arriving activation and updates the top model without waiting for all participating clients.
Meanwhile, the client-side bottom models are periodically aggregated according to Eq.~\eqref{eq:conventional_aggregation} to establish a shared global state.
We define the interval between two consecutive bottom model aggregations as an aggregation round.
Within each round, top model updates are performed asynchronously, whereas the bottom models are synchronized at the round boundary.
This execution pattern introduces inconsistency along two structural dimensions: horizontally across clients and vertically across the clients and server.

\begin{figure}[t]
    \centering
    \subfloat[Label divergence at the PS.\label{fig:horizontal_label}]{
        \includegraphics[width=0.45\linewidth,height=3.3cm]{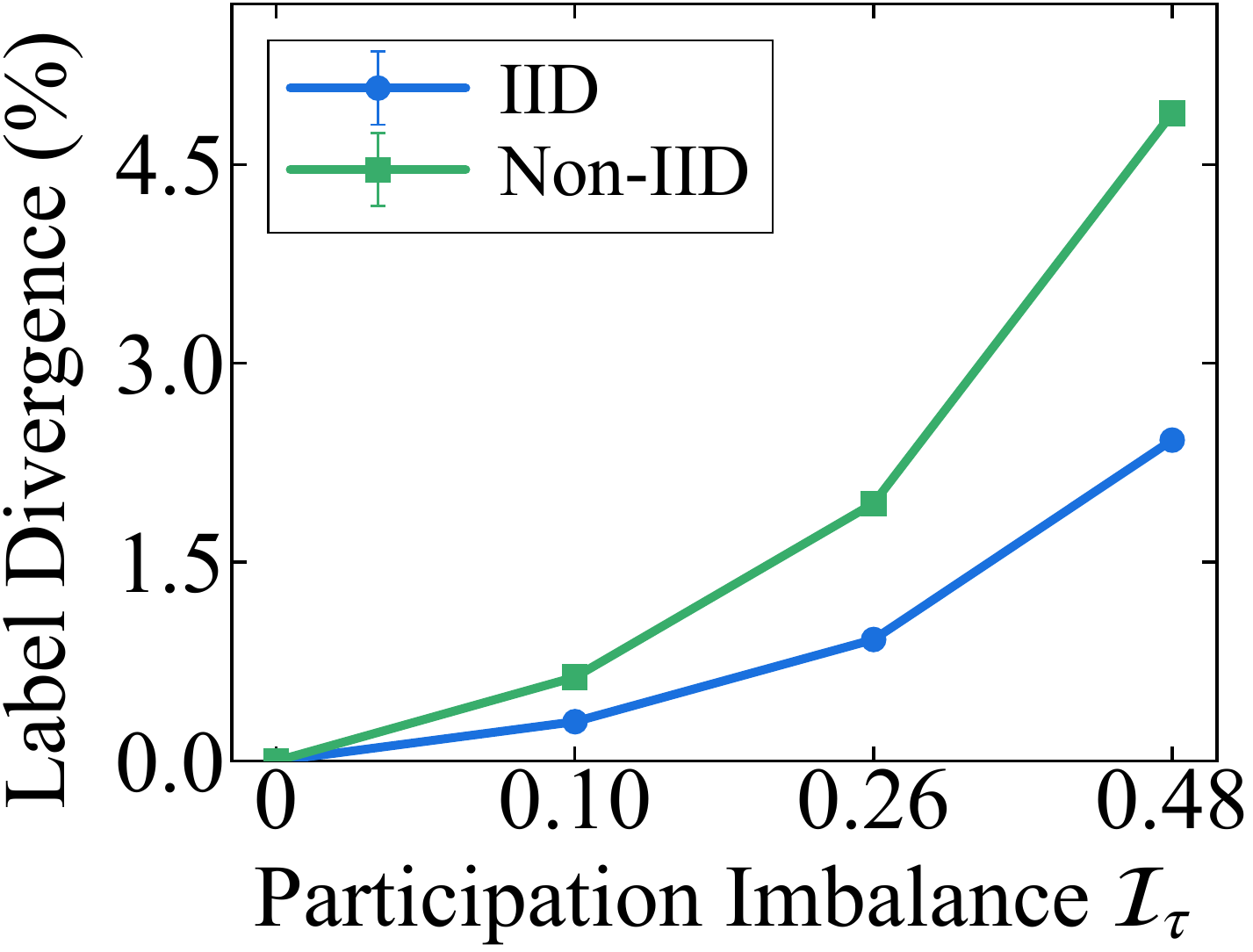}
    }
    \hfill
    \subfloat[Top-update direction deviation.\label{fig:horizontal_direction}]{
        \includegraphics[width=0.45\linewidth,height=3.3cm]{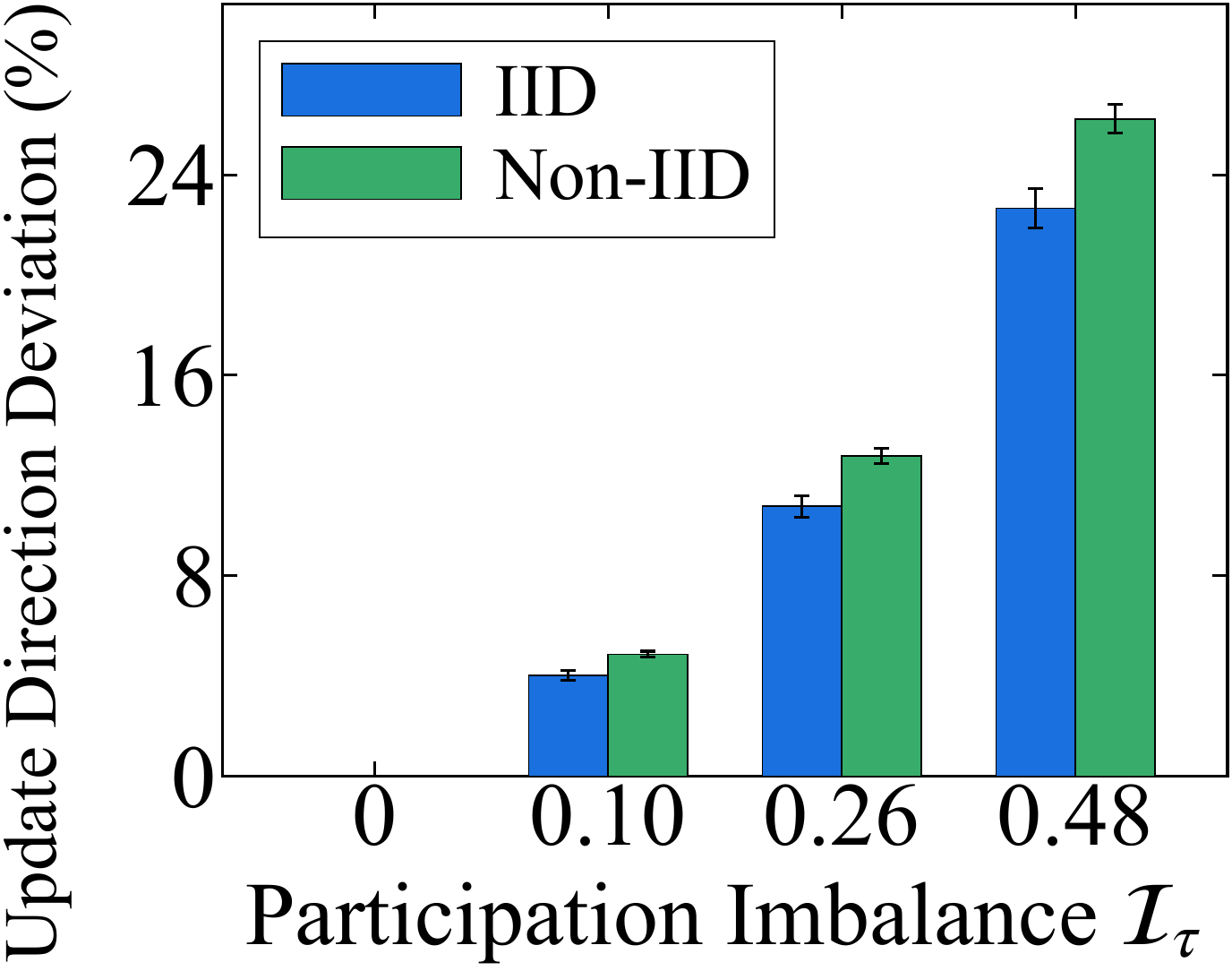}
    }
    \caption{Impact of horizontal inconsistency.}
    \label{fig:impact_of_horizontal_inconsistency}
\end{figure}

\subsection{Characterization of Horizontal Inconsistency}
\label{sec:impact_horizontal_inconsistency}


In asynchronous SFL, heterogeneous client computation and communication capabilities cause activations to arrive at the PS at different rates.
Since each arriving activation immediately triggers a top model update, faster clients participate in substantially more updates than slower ones, leading to \textbf{horizontal inconsistency}.
Under IID data, this inconsistency mainly changes the relative contribution of clients.
Under non-IID data, however, repeatedly arriving activations from the same clients skew the effective activation and label distributions observed by the PS, causing the optimization direction of the top model to gradually deviate from the global objective.


To characterize horizontal inconsistency, we conduct a controlled experiment using DeiT on CIFAR-100.
The model is split at the 6th layer and trained with 20 clients under both IID and non-IID settings.
Since unequal client influence is directly reflected by their participation frequencies, we use participation imbalance as the measurable indicator in this experiment.
Let $\tau_i^h$ denote the number of top model updates involving client $i$ during aggregation round $h$, and quantify participation imbalance as:
\begin{equation}
\mathcal{I}_{\tau}^{h}
=
1-
\frac{
\left(\sum_{i=1}^{N}\tau_i^h\right)^2
}{
N\sum_{i=1}^{N}\left(\tau_i^h\right)^2
}
\label{eq:participation_imbalance}
\end{equation}
Here, $\mathcal{I}_{\tau}^{h}=0$ indicates balanced participation, whereas a larger value represents a more imbalanced participation pattern in this controlled experiment.
To assess whether the top model learns evenly from different client contributions, we examine the divergence between the label distribution in activations at the PS and that under balanced client participation.
To quantify the bias in top model learning direction, we measure the deviation from the balanced top model update direction.
Meanwhile, we keep the model state, sampled batches, and client-specific update directions unchanged, so the measured differences are caused solely by unequal client participation.

As shown in Fig.~\ref{fig:horizontal_label}, as $\mathcal{I}_{\tau}^{h}$ increases, the distribution of labels associated with the activations processed by the PS increasingly deviates from the reference distribution under balanced client participation.
Under the same participation imbalance, non-IID data produce nearly twice the divergence of IID data, indicating that statistical heterogeneity amplifies the distribution bias.
Fig.~\ref{fig:horizontal_direction} further shows that the update direction deviation of the top model increases from nearly zero under balanced participation to approximately $23\%$ and $26\%$ at the highest imbalance level for IID and non-IID data, respectively.
These results demonstrate that unequal client participation systematically biases data distribution and update direction toward frequently participating clients, \ie, horizontal inconsistency.
As they accumulate, uneven learning across clients and biased update directions slow model convergence and reduce final accuracy, especially under non-IID data.
Therefore, the server should coordinate arriving activations to preserve activation diversity instead of updating the top model solely according to arrival order.


\begin{figure}[t]
    \centering
    \subfloat[Split layer gradient deviation.\label{fig:vertical_gradient}]{
        \includegraphics[width=0.45\linewidth,height=3.3cm]{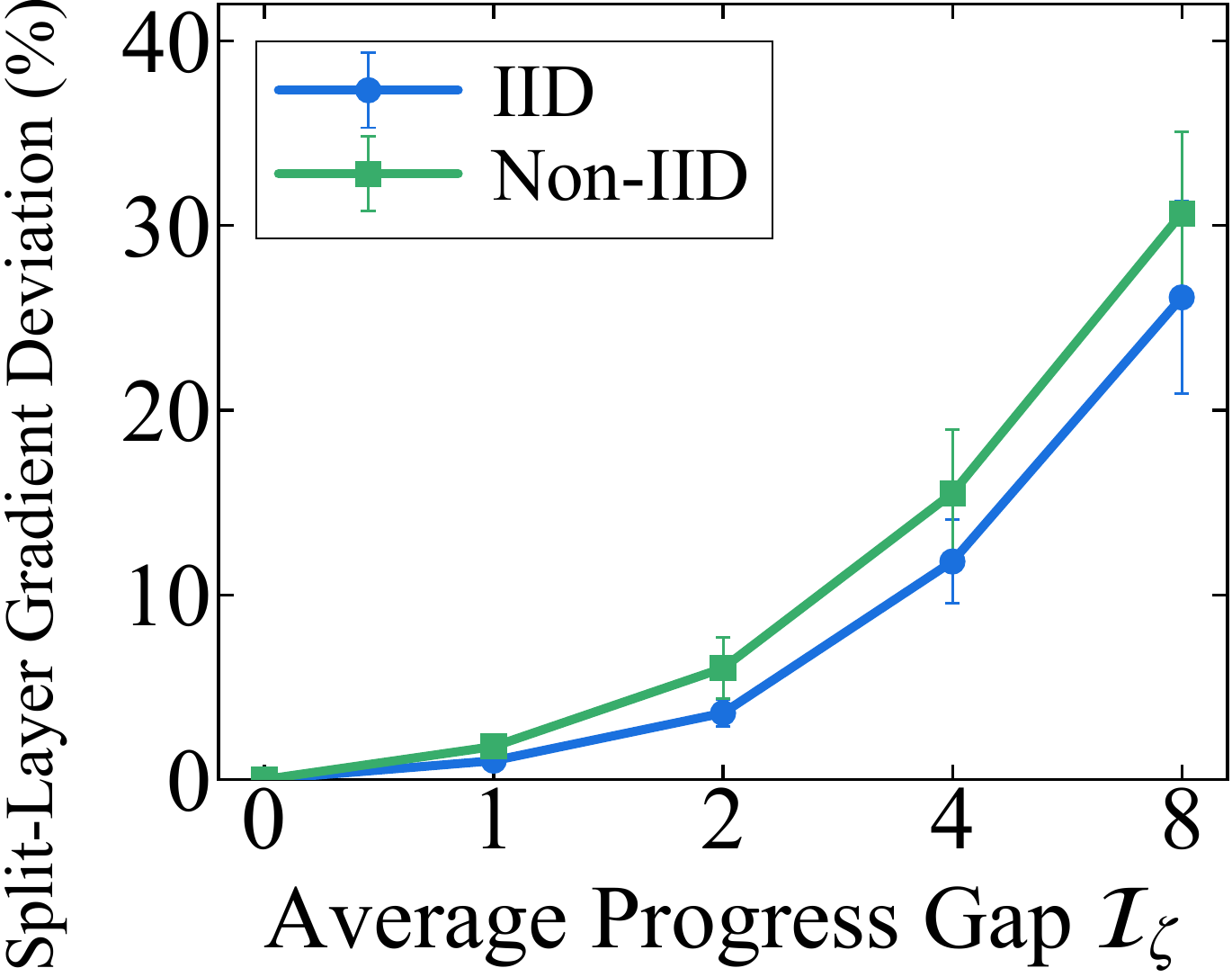}
    }
    \hfill
    \subfloat[Model aggregation deviation.\label{fig:vertical_aggregation}]{
        \includegraphics[width=0.45\linewidth,height=3.3cm]{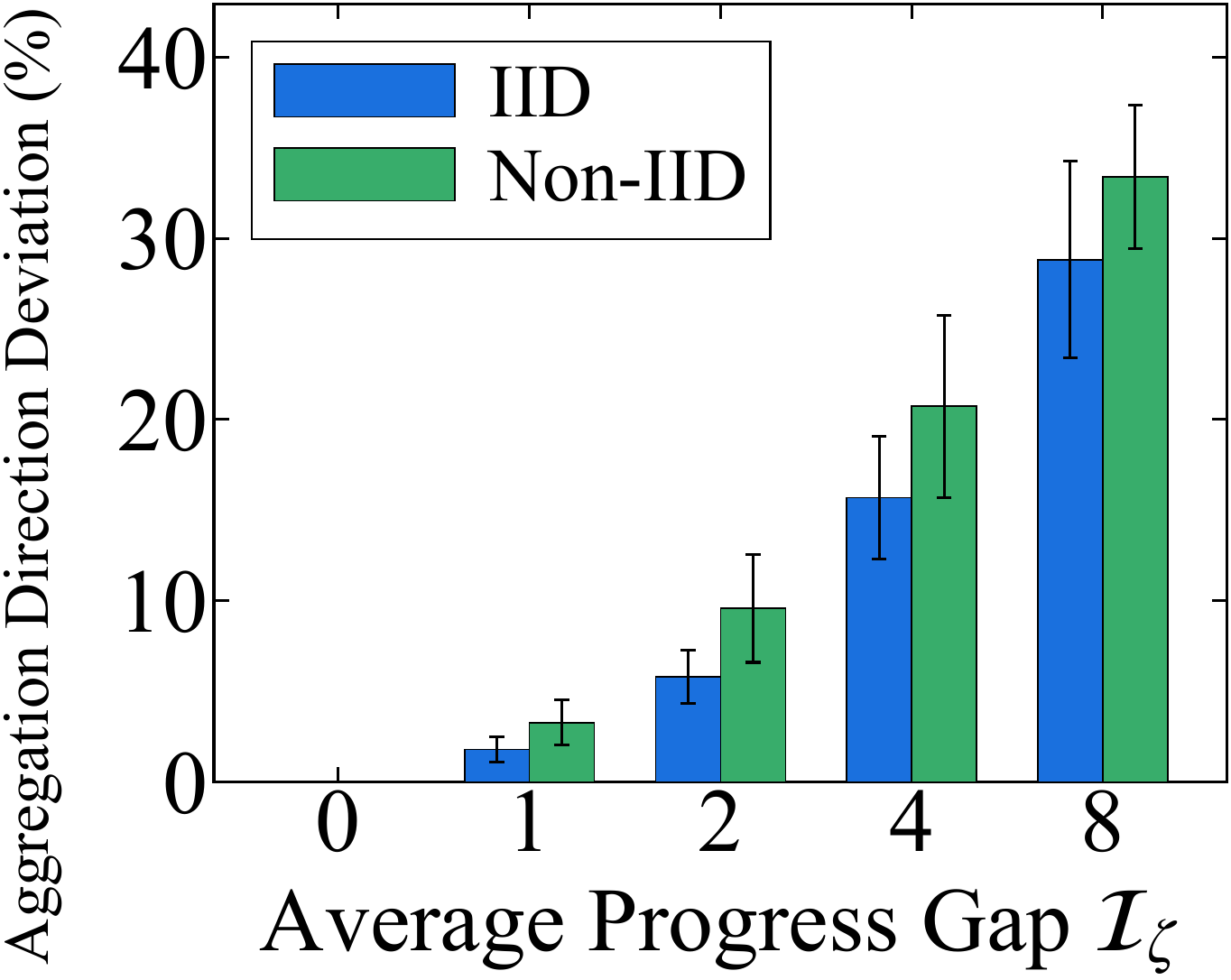}
    }
    \caption{Impact of vertical inconsistency.}
    \label{fig:impact_of_vertical_inconsistency}
\end{figure}


\subsection{Characterization of Vertical Inconsistency}
While horizontal inconsistency biases server-side optimization, asynchronous SFL also introduces \textbf{vertical inconsistency} between the bottom model on each client and the top model on the PS.
The top model advances whenever an activation is processed, whereas each bottom model progresses only through its own client interactions. 
Consequently, different clients may exhibit different progress gaps between their bottom models and the current top model.
As this progress gap increases, the gradient exchanged at the split layer deviates more from that under aligned model partitions, further affecting the corresponding bottom model update.
Aggregating bottom models updated under different partition states may therefore bias the global bottom model update direction.


Using the same experimental setup as Sec.~\ref{sec:impact_horizontal_inconsistency}, for each evaluated client $i$ participating in round $h$, we define its average interaction progress gap as:
\begin{equation}
    \zeta_i^h
    =
    \frac{1}{\tau_i^h}
    \sum_{k=1}^{\tau_i^h}
    \left(r_{i,k}^h - k\right)
    \label{eq:progress_gap}
\end{equation}
where $r_{i,k}^h$ denotes the top model update index corresponding to the $k$-th bottom model update of client $i$, and a larger $\zeta_i^h$ indicates that its bottom model updates are on average performed under a greater progress mismatch between the two model partitions.
To consistently quantify vertical inconsistency, we define its level as the average progress gap over the set $\mathcal{U}$ of evaluated bottom models :
\begin{equation}
\mathcal{I}_{\zeta}^{h}
=
\frac{1}{|\mathcal{U}|}
\sum_{i\in\mathcal{U}}\zeta_i^h
\label{eq:vertical_inconsistency}
\end{equation}
where $\mathcal{I}_{\zeta}^{h}=0$ corresponds to fully aligned model partitions.
To characterize vertical inconsistency, we measure the split layer gradient deviation and aggregation direction deviation, which quantify the difference in update direction and gradient mismatch of the bottom model, respectively.
We vary $\mathcal{I}_{\zeta}^{h}$ over $\{0,1,2,4,8\}$.
For the split layer gradient experiment, the same input batch is processed using bottom model states corresponding to different progress gaps, while keeping the current top model and labels fixed.
For the aggregation experiment, client groups are assigned evenly spaced progress gaps whose mean equals $\mathcal{I}_{\zeta}^{h}$, and the resulting aggregated update direction is compared with that under fully aligned model partitions.

As shown in Fig.~\ref{fig:vertical_gradient}, the split layer gradient deviation consistently increases with $\mathcal{I}_{\zeta}^{h}$, reaching approximately $26\%$ under IID data and $30\%$ under non-IID data when $\mathcal{I}_{\zeta}^{h}=8$.
This indicates that a larger progress mismatch causes the returned gradient, and hence the bottom model update direction, to deviate further from that under aligned model partitions.
Fig.~\ref{fig:vertical_aggregation} further shows that such gradient mismatch propagates to bottom model aggregation, where the aggregation direction deviation reaches approximately $29\%$ and $33\%$ under IID and non-IID data, respectively.
The consistently larger deviations under non-IID data indicate that statistical heterogeneity amplifies the impact of vertical inconsistency.
As these inconsistencies accumulate across successive updates and aggregation rounds, they can cause the global model to deviate from the desired optimization trajectory, thereby slowing convergence and
reducing final accuracy.

\subsection{Discussion}

Although horizontal and vertical inconsistencies are analyzed separately above, they are tightly coupled during asynchronous SFL. 
Horizontal inconsistency biases the optimization direction of the top model, while the resulting server-side model evolution further enlarges the progress mismatch experienced by delayed clients.
In turn, bottom models updated under different partition states generate less consistent activations in subsequent rounds, further amplifying the server-side optimization bias caused by unequal client influence. 
Therefore, mitigating only one type of inconsistency cannot fundamentally stabilize asynchronous SFL.
Instead, activation-level consistency and model-level consistency should be jointly optimized in a unified framework.
Motivated by this observation, we propose \method, which integrates dynamic activation merging and consistency-aware model aggregation into a coupled consistency optimization framework.

\section{System Design}\label{sec:design}
\begin{figure}[t]
    \centering
    \includegraphics[width=\linewidth]{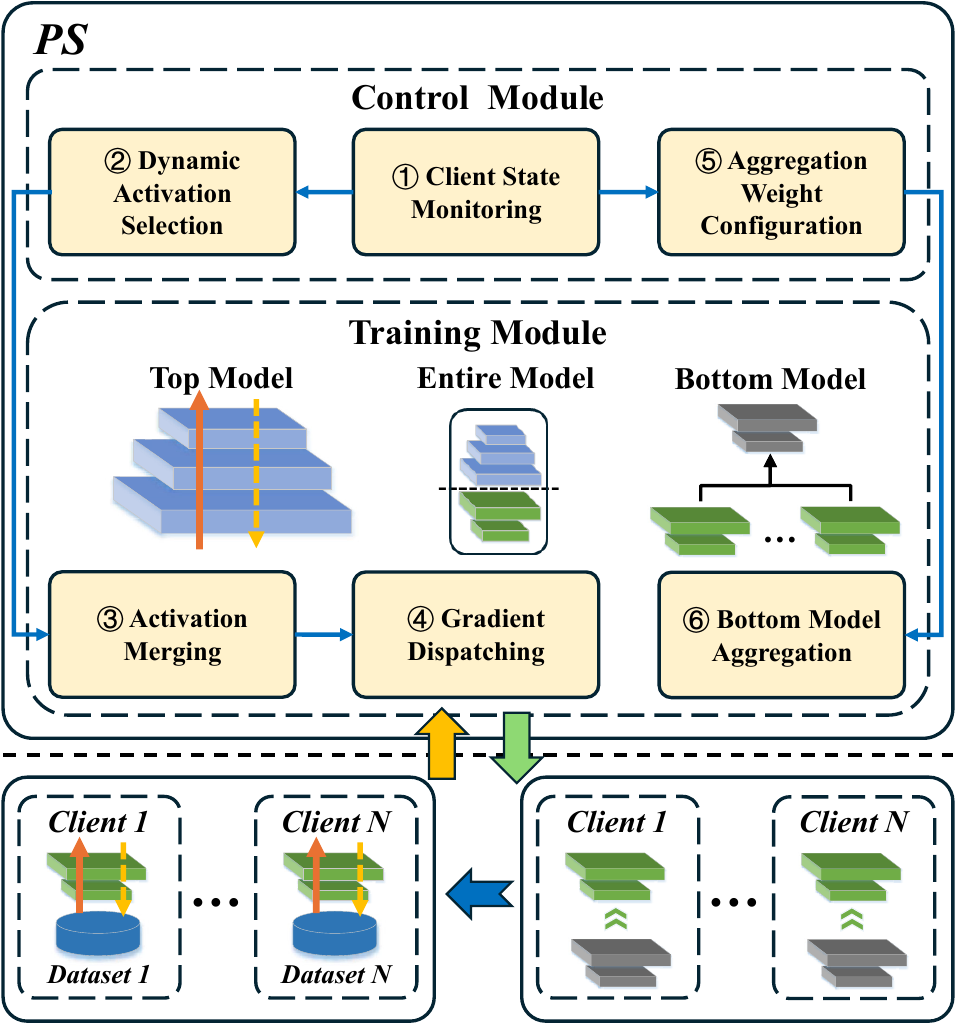}
    \caption{Overview of \method}
    \label{fig:AsySFL_overview}
\end{figure}

\subsection{Overview}
As illustrated in Fig.~\ref{fig:AsySFL_overview}, CoASFL consists of a control module and a training module that collaboratively optimize asynchronous split federated learning.
Specifically, the PS continuously monitors the training state of each client, including its participation frequency, progress gap, and runtime status throughout the training process (\circled{1}).
Based on the maintained client states, the PS determines a set of compatible activations from the activation buffer before each top model update (\circled{2}).
Then the selected activations are merged to optimize the top model (\circled{3}), and the corresponding split layer gradients are dispatched to the associated clients for bottom model optimization (\circled{4}).
Meanwhile, the monitored client states are continuously updated and used to determine the contribution of each client during model synchronization (\circled{5}).
Finally, the PS aggregates the bottom models using the configured weights and broadcasts the updated bottom model to all clients, initiating the next aggregation round (\circled{6}).
By jointly coordinating activation scheduling during top model optimization and adaptive aggregation during bottom model synchronization, CoASFL simultaneously mitigates horizontal and vertical inconsistencies in asynchronous SFL.

\subsection{Control Module}
During each aggregation round, the control module of the PS monitors client states and makes the decisions required by the training module.
Specifically, it determines the activation set for each top model update and configures the bottom model aggregation weights at the round boundary.

\subsubsection{\textbf{Client State Monitoring}}
The control module maintains the client information required for subsequent control decisions.
Specifically, the label distributions and computation and communication capabilities of all clients are collected to support dynamic activation selection.
Meanwhile, their participation frequencies and progress gaps are continuously maintained during each round for bottom model aggregation weight configuration at the round boundary.
Let $\mathbf{V}=\{v_j\geq 0, j\in [1,M]\}$ denote a categorical distribution of class labels over $M$ classes, where $\sum_{j=1}^{M}v_j=1$.
Since labels are transmitted alongside activations for loss computation in conventional SFL \cite{thapa2022splitfed}, the PS collects the empirical label distribution $\mathbf{V}_i$ of each client $i$ once before training and uses it to evaluate whether a client combination approaches the balanced label distribution.
Considering that privacy leakage is an important challenge in SFL, popular privacy protection techniques, \eg, Differential Privacy \cite{thapa2021advancements,wu2023federated}, Distance Correlation \cite{vepakomma2019reducing} and Homomorphic Encryption \cite{yang2023dynamic}, can be applied to protect the raw data, activations/gradients and models, and are orthogonal to the main focus of \method.

The time-varying computation and communication states are used to estimate how long a buffered activation may wait for a suitable combination.
As proxy metrics, we adopt the computing time of one iteration as $\mu^h_{i}$ and the corresponding transmission time $\beta^h_{i}$, which can be recorded by the clients directly during model training, to indicate the computing and communication capabilities of client $i$ in round $h$, respectively.
Before starting model training in round $h$, the PS collects the latest computing time $\hat{\mu}^h_i$ and transmission time $\hat{\beta}^h_i$ from all clients.
Additionally, we utilize the moving average with historical client states to enhance the robustness of the estimation \cite{leroy2019federated}.
Accordingly, the PS estimates the computing time $\mu^h_i$ and corresponding transmission time $\beta^h_i$ for client $i$ in round $h$ by calculating the moving average with $\alpha\in[0,1]$ (\eg, $0.8$ in our experiments) as:
\begin{align}
    \mu^h_i=\alpha\cdot\mu^{h-1}_i + (1-\alpha)\cdot\hat{\mu}^h_i \\
    \beta^h_i=\alpha\cdot\beta^{h-1}_i + (1-\alpha)\cdot\hat{\beta}^h_i 
\end{align}
\noindent Besides, it is worth noting that advancing the techniques for state estimation is not our focus, and other existing advanced estimation techniques \cite{halperin2010predictable, yue2017linkforecast} can be applied.

The training states $\tau_i^h$ and $\zeta_i^h$ are continuously updated throughout each aggregation round.
The participation frequency $\tau_i^h$ guides dynamic activation selection toward less frequently participating clients, while $\tau_i^h$ and $\zeta_i^h$ are jointly used to configure the bottom model aggregation weights at the round boundary.

\subsubsection{\textbf{Dynamic Activation Selection}}

As demonstrated in Sec.~\ref{sec:impact_horizontal_inconsistency}, updating the top model strictly according to activation arrival order biases the training sequence toward frequently participating clients.
To alleviate this bias, the PS buffers arriving activations and selects a subset before each top model update according to their label distributions and participation frequencies.
Let $\Theta^r$ denote the set of clients whose activations are pending before the $r$-th top model update.
Before each top model update, the control module selects a client subset $S^r\subseteq\Theta^r$ from the buffered activations.
Exhaustively searching the subsets of $\Theta^r$ online for each update is, however, computationally prohibitive.
We therefore pre-construct valid candidate combinations offline from the client label distributions before training.


To reduce the label-distribution bias under non-IID data, each candidate combination should represent a label distribution close to that under balanced client participation.
Because all clients use the same fixed batch size, we define the balanced label distribution as $\Phi_0=\frac{1}{N}\sum_{i=1}^{N}\mathbf{V}_i$.
For a client combination $C\subseteq\{1,\ldots,N\}$, its label distribution is $\Phi_C=\frac{1}{|C|}\sum_{i\in C}\mathbf{V}_i$.
We use the KL divergence to measure the gap between $\Phi_C$ and $\Phi_0$ as follows \cite{goldberger2003efficient,4218101}:
\begin{equation}
    \mathrm{KL}(\Phi_C\|\Phi_0)
    =
    \sum_{j=1}^{M}
    \Phi_C(v_j)
    \log
    \frac{\Phi_C(v_j)}{\Phi_0(v_j)}
    \label{eq:distribution_divergence}
\end{equation}
A smaller value indicates that the activations from $C$ represent a label distribution closer to that under balanced client participation.
Consequently, a combination is regarded as valid if it satisfies the constraints:
\begin{equation}
\left\{
\begin{array}{l}
    \mathrm{KL}(\Phi_C\|\Phi_0)\leq\epsilon\\
    |C|\leq\Theta_0
\end{array}
\right.
\label{eq:valid_combination_constraints}
\end{equation}
where $\epsilon\geq0$ is a predefined distribution threshold, and $\Theta_0$ denotes the activation buffer capacity.
We set $\epsilon=0.25$ and $\Theta_0=N/5$ in our experiments, under which each client belongs to at least one valid candidate combination.










Given the exponential growth of possible client combinations with respect to $N$, we adopt a greedy algorithm to construct the valid candidate set $\mathcal{Y}$.
Starting from each client as a seed, the procedure iteratively adds the unselected client that yields the smallest KL divergence after expansion.
Each intermediate combination satisfying Eq.~\eqref{eq:valid_combination_constraints} is retained, and the expansion process terminates once its size reaches $\Theta_0$. 
Each seed requires at most $\Theta_0-1$ successive expansions, yielding an overall construction complexity of $O(N^2\Theta_0M)$ rather than the exponential cost of exhaustive enumeration.
This construction is performed only once before training and produces at most $N\Theta_0$ candidate combinations, which are then reused during online activation selection.

During online training, the PS obtains the available candidate set $\mathcal{A}^r = \{C \in \mathcal{Y} \mid C \subseteq \Theta^r\}$ from the buffered activations.
When $\mathcal{A}^r \neq \emptyset$, the control module selects a combination $S^r$ by minimizing the average participation frequency:
\begin{equation}
    S^r
    =
    \arg\min_{C \in \mathcal{A}^r}
    \frac{1}{|C|}
    \sum_{i \in C} \tau_i^h
    \label{eq:activation_selection}
\end{equation}
This objective prioritizes combinations containing clients with lower participation frequencies, thereby progressively balancing client influence during the round, with ties broken by the smaller KL divergence.
When no candidate is available, the buffered activations are retained for subsequent matching.

However, this matching process may cause an activation to remain buffered for an extended duration when the other members of its candidate combinations have yet to arrive. 
To prevent indefinite buffering, we introduce a client-specific timeout threshold derived from the estimated computation and communication states. 
For client $i$ and a valid combination $C \in \mathcal{Y}$ containing $i$, the expected matching delay is the difference between the latest expected arrival time among the members of $C$ and that of client $i$ itself, \ie, $\max_{j \in C}(\mu_j^h + \beta_j^h) - (\mu_i^h + \beta_i^h)$. 
Following the practice of percentile-based tail-latency control in large-scale systems \cite{beyer2016site}, we set the timeout threshold $\lambda_i^h$ to the P90 value of these delays across all $C \in \mathcal{Y}$ with $i \in C$, and update it once per aggregation round.
Let $q_i^r$ denote the elapsed waiting time of the pending activation from client $i$. 
When one or more activations exceed their thresholds, the client with the largest timeout excess is identified as $i^\star$ according to:
\begin{equation}
i^\star = \arg\max_{i: q_i^r \geq \lambda_i^h} (q_i^r - \lambda_i^h) 
\label{eq:largest_timeout}
\end{equation}
The PS then restricts the candidate set to those containing $i^\star$.
The control module applies Eq.~\eqref{eq:activation_selection} to the remaining candidates.
If no available candidate contains $i^\star$, it greedily constructs a combination from the pending activations starting with $\{i^\star\}$, each time adding the client that yields the smallest KL divergence until a valid combination is formed or the combination size reaches $\Theta_0$. 
The expansion terminates when a valid combination is formed.
If validity remains unattainable, the intermediate combination with the smallest KL divergence is selected.
When the buffer reaches capacity without any valid candidate and no timeout has occurred, the longest-waiting client is treated as $i^\star$ and processed through the same greedy construction procedure.

\subsubsection{\textbf{Aggregation Weight Configuration}}

At the end of aggregation round $h$, the bottom models of different clients have incorporated different numbers of updates, and these updates have experienced different average progress gaps.
With the same fixed batch size, the conventional sample-size-based aggregation weight $A_i^h$ is proportional to the participation frequency $\tau_i^h$.
Although this weight reflects the amount of training completed by each client, it does not correct the contribution imbalance caused by unequal participation or account for the update deviation induced by cross-partition progress mismatch.
The control module thus reconfigures the aggregation weights using both $\tau_i^h$ and $\zeta_i^h$.

The two states characterize complementary properties of a bottom model.
Specifically, $\tau_i^h$ reflects the number of updates incorporated into the bottom model, whereas $\zeta_i^h$ reflects the average progress mismatch experienced during these updates.
Therefore, the correction should not compensate low-participation clients indiscriminately. 
Instead, participation compensation should be conditioned on the consistency of the updates incorporated into each bottom model, \ie,  a client with fewer but better-aligned updates receives stronger compensation, whereas a client whose updates experience a large progress mismatch should be discounted even if frequency is low.
We adopt a multiplicative form, which compresses the contribution disparity from unequal participation, decreases with the average progress gap, and reduces to one under balanced participation and fully aligned model partitions.


Let $\bar{\tau}^h=\frac{1}{N}\sum_{j=1}^{N}\tau_j^h$ denote the average participation frequency, and let $R^h$ denote the total number of top model updates in round $h$.
The correction factor for client $i$ is configured as:
\begin{equation}
    \xi_i^h
    =
    \frac{\bar{\tau}^h+1}{\tau_i^h+1}
    \cdot
    \frac{1}{1+\zeta_i^h/R^h}
    \label{eq:aggregation_correction_factor}
\end{equation}
The first term compensates the contribution loss caused by low participation relative to the round average. 
The second term acts as a consistency gate that weakens this compensation as the progress mismatch increases. 
Their multiplicative combination therefore jointly considers the horizontal contribution imbalance and the vertical reliability of bottom model updates.
Normalizing $\zeta_i^h$ by $R^h$ makes the correction comparable across rounds with different numbers of top model updates.
The additive constant in the participation term avoids division by zero and smooths the correction for small participation counts, while the leading one in the progress-gap term keeps the corresponding
factor neutral when $\zeta_i^h=0$.

The final aggregation weight is obtained by applying the correction factor to the base weight and normalizing the results across all clients, ensuring that the configured weights remain a valid convex combination:
\begin{equation}
    \widetilde{A}_i^h
    =
    \frac{A_i^h\xi_i^h}
    {\sum_{j=1}^{N}A_j^h\xi_j^h}
    \label{eq:configured_aggregation_weight}
\end{equation}
This normalization guarantees
$\widetilde{A}_i^h\geq0$ and
$\sum_{i=1}^{N}\widetilde{A}_i^h=1$.
When all clients have the same participation frequency and their bottom model updates experience no progress mismatch, $\xi_i^h=1$ for all clients, and the configured weights reduce to the conventional aggregation weights.


\subsection{Training Module}

The training module executes model training and bottom model synchronization according to the decisions made by the control module.
It consists of activation merging, gradient dispatching, and bottom model aggregation and distribution.

\subsubsection{\textbf{Activation Merging}}

For the $r$-th top model update, once the control module determines the activation set $S^r$, the PS retrieves the corresponding activations from the buffer and merges them into a mixed activation
sequence $G^r= [g_i^r\mid i\in S^r]$.
The PS then performs forward and backward propagation over $\boldsymbol{G}^r$ and updates the top model as follows:
\begin{equation}
    \boldsymbol{w}_p^{r+1}
    =
    \boldsymbol{w}_p^r
    -
    \eta
    \widetilde{\nabla}
    F_p(\boldsymbol{w}_p^r;G^r)
    \label{eq:top_model_update}
\end{equation}
where $\widetilde{\nabla}F_p(\boldsymbol{w}_p^r; G^r)$  denotes the gradient averaged over the samples represented by $G^r$.

\subsubsection{\textbf{Gradient Dispatching}}

After backward propagation, the PS obtains the mixed split layer gradient $\widehat{G}^r$ corresponding to $G^r$.
Based on the client-wise composition of $G^r$, the PS separates it as $\widehat{G}^r= [\widehat{g}_i^r\mid i\in S^r]$.
It then dispatches each separated gradient $\widehat{g}_i^r$ to the corresponding client.
Upon receiving the gradient, client $i$ completes backward propagation and updates its bottom model according to Eq.~\eqref{eq:sfl_bottom_update}.

\subsubsection{\textbf{Bottom Model Aggregation}}

Aggregation round $h$ terminates once every client has completed at least one update and the predefined average local update frequency is reached. 
After all pending gradients have been applied to the corresponding bottom models, the PS collects and aggregates them using the weights configured by the control module:
\begin{equation}
    \boldsymbol{w}_b^{h+1}
    =
    \sum_{i=1}^{N}
    \widetilde{A}_i^h
    \boldsymbol{w}_{b,i}^{h}
    \label{eq:work_bottom_aggregate}
\end{equation}
The PS subsequently distributes $\boldsymbol{w}_b^{h+1}$ to all clients for the next aggregation round, or be combined with the top model to form a complete model.

\section{Convergence Analysis}\label{sec:conv}
Since the training objective is generally nonconvex, we analyze the mean-square gradients of the bottom and top models over $H$ aggregation rounds to characterize their convergence behavior.
Following existing convergence analyses~\cite{liao2023accelerating,han2021accelerating,yang2021achieving,xu2022adaptive}, we adopt three standard assumptions.

\begin{assumption}\label{ass:coasfl_smooth}
\textit{(Lipschitz Continuous Gradient).}
The balanced objective $F$ and the bottom and top model gradients are $L$-smooth, satisfying the following inequalities:
\begin{equation}\label{eq:ca_smooth}
\begin{aligned}
\|\nabla F(\boldsymbol x)-\nabla F(\boldsymbol y)\|
 &\le L\|\boldsymbol x-\boldsymbol y\|\\
\|\nabla F_{b,i}(\boldsymbol x)-\nabla F_{b,i}(\boldsymbol y)\|
 &\le L\|\boldsymbol x-\boldsymbol y\|\\
\|\nabla F_{p,i}(\boldsymbol x)-\nabla F_{p,i}(\boldsymbol y)\|
 &\le L\|\boldsymbol x-\boldsymbol y\|
\end{aligned}
\end{equation}
\end{assumption}

\begin{assumption}\label{ass:coasfl_unbiased}
\textit{(Unbiased Local Gradient Estimator).}
The stochastic gradients are unbiased estimators of the corresponding bottom and top model gradients, expressed as follows:
\begin{equation}\label{eq:ca_unbiased}
\begin{aligned}
\mathbb E[\widetilde\nabla F_{b,i}(\boldsymbol w)]
 &=\nabla F_{b,i}(\boldsymbol w)\\
\mathbb E[\widetilde\nabla F_{p,i}(\boldsymbol w)]
 &=\nabla F_{p,i}(\boldsymbol w)
\end{aligned}
\end{equation}
\end{assumption}

\begin{assumption}\label{ass:coasfl_moment}
\textit{(Bounded Gradient Second Moment).}
The stochastic gradients of both model partitions have bounded second moments, satisfying:
\begin{equation}\label{eq:ca_moment}
\begin{aligned}
\mathbb E\|\widetilde\nabla F_{b,i}(\boldsymbol w)\|^2
 &\le\sigma_b^2\\
\mathbb E\|\widetilde\nabla F_{p,i}(\boldsymbol w)\|^2
 &\le\sigma_p^2
\end{aligned}
\end{equation}
\end{assumption}

Let $\boldsymbol w^h=[\boldsymbol w_b^h,\boldsymbol w_p^{h,0}]$ denote the joint model state at the beginning of round $h$, and let $S^{h,r}$ denote the activation set selected for its $r$-th top model update.
According to Eqs.~\eqref{eq:aggregation_correction_factor}--\eqref{eq:work_bottom_aggregate}, the normalized gradient contributions can be expressed as follows:
\begin{equation}\label{eq:ca_weights}
\begin{aligned}
\pi_i^h&=\frac1{R^h}\sum_{r=1}^{R^h}
 \frac{\mathbf1\{i\in S^{h,r}\}}{|S^{h,r}|}\\
\gamma^h&=\sum_{i=1}^N\widetilde A_i^h\tau_i^h\\
\rho_i^h&=\frac{\widetilde A_i^h\tau_i^h}{\gamma^h}
\end{aligned}
\end{equation}
Here, $\sum_i\pi_i^h=\sum_i\rho_i^h=1$.
Let $\mathbb E_h$ denote expectation conditioned on $\boldsymbol w^h$ and the execution schedule, taken to be independent of stochastic gradient sampling.

\begin{lemma}\label{lem:ca_direction}
Under Assumptions~\ref{ass:coasfl_smooth}--\ref{ass:coasfl_moment}, the expected update direction satisfies the following inequality:
\begin{equation}\label{eq:ca_direction}
\begin{aligned}
&\mathbb E_h\langle\nabla F(\boldsymbol w^h),\boldsymbol d^h\rangle\\
&\quad\ge\frac{\gamma^h}{2}
 \Big[\|\nabla F_b(\boldsymbol w^h)\|^2-(B_b^h)^2\Big]\\
&\qquad+\frac{R^h}{2}
 \Big[\|\nabla F_p(\boldsymbol w^h)\|^2-(B_p^h)^2\Big]
\end{aligned}
\end{equation}
where $\boldsymbol d^h=(\boldsymbol w^h-\boldsymbol w^{h+1})/\eta$, and $\nabla F_b=N^{-1}\sum_i\nabla F_{b,i}$ and $\nabla F_p=N^{-1}\sum_i\nabla F_{p,i}$ denote the two components of $\nabla F$.
The corresponding gradient deviation bounds are defined as follows:
\begin{equation}\label{eq:ca_bias}
\begin{aligned}
B_p^h={}&\left\|\sum_i\left(\pi_i^h-\frac1N\right)
 \nabla F_{p,i}(\boldsymbol w^h)\right\|\\
&+L\eta\left[\sigma_b\sum_i\pi_i^h(\tau_i^h-1)
 +\frac{\sigma_p(R^h-1)}2\right]\\
B_b^h={}&\left\|\sum_i\left(\rho_i^h-\frac1N\right)
 \nabla F_{b,i}(\boldsymbol w^h)\right\|\\
&+L\eta\left[\sigma_p\sum_i\rho_i^h\zeta_i^h
 +\frac{\sigma_b+\sigma_p}{2}
 \sum_i\rho_i^h(\tau_i^h-1)\right]
\end{aligned}
\end{equation}
\end{lemma}

\begin{IEEEproof}
According to Eqs.~\eqref{eq:sfl_bottom_update}, \eqref{eq:top_model_update}, and \eqref{eq:work_bottom_aggregate}, the aggregated update direction can be expressed as follows:
\begin{equation}\label{eq:ca_actual}
\begin{aligned}
\boldsymbol d^h
 &=\big[\gamma^h\overline{\boldsymbol g}_b^h,
 R^h\overline{\boldsymbol g}_p^h\big]\\
\overline{\boldsymbol g}_b^h
 &=\sum_i\frac{\rho_i^h}{\tau_i^h}
 \sum_{k=1}^{\tau_i^h}\widetilde{\boldsymbol g}_{b,i}^{h,k}\\
\overline{\boldsymbol g}_p^h
 &=\frac1{R^h}\sum_{r=1}^{R^h}
 \widetilde{\boldsymbol g}_p^{h,r}
\end{aligned}
\end{equation}
where $\widetilde{\boldsymbol g}_{b,i}^{h,k}$ and $\widetilde{\boldsymbol g}_p^{h,r}$ denote the corresponding stochastic gradients.
Since all clients begin each round with the same bottom model and perform sequential local updates, Assumption~\ref{ass:coasfl_moment} gives the following parameter deviation bounds:
\begin{equation}\label{eq:ca_state_drift}
\begin{aligned}
\mathbb E_h\|\boldsymbol w_{b,i}^{h,k-1}-\boldsymbol w_b^h\|
 &\le\eta\sigma_b(k-1)\\
\mathbb E_h\|\boldsymbol w_p^{h,r-1}-\boldsymbol w_p^{h,0}\|
 &\le\eta\sigma_p(r-1)
\end{aligned}
\end{equation}
According to Assumptions~\ref{ass:coasfl_smooth}--\ref{ass:coasfl_unbiased}, the expected top model gradient deviation can be bounded as follows:
\begin{equation}\label{eq:ca_topdrift}
\begin{aligned}
&\left\|\mathbb E_h[\overline{\boldsymbol g}_p^h]
 -\sum_i\pi_i^h\nabla F_{p,i}(\boldsymbol w^h)\right\|\\
&\quad\le\frac{L\eta}{R^h}\sum_{r=1}^{R^h}
 \left[\frac{\sigma_b}{|S^{h,r}|}
 \sum_{i\in S^{h,r}}(k_i^{h,r}-1)
 +\sigma_p(r-1)\right]\\
&\quad\le L\eta\left[\sigma_b\sum_i\pi_i^h(\tau_i^h-1)
 +\frac{\sigma_p(R^h-1)}2\right]
\end{aligned}
\end{equation}
where $k_i^{h,r}$ denotes the local update index associated with client $i$ in $S^{h,r}$.
Applying the same analysis to the bottom model gradients, we obtain the corresponding deviation bound:
\begin{equation}\label{eq:ca_bottomdrift}
\begin{aligned}
&\left\|\mathbb E_h[\overline{\boldsymbol g}_b^h]
 -\sum_i\rho_i^h\nabla F_{b,i}(\boldsymbol w^h)\right\|\\
&\quad\le L\eta\sum_i\frac{\rho_i^h}{\tau_i^h}
 \sum_{k=1}^{\tau_i^h}
 \left[\sigma_b(k-1)+\sigma_p(r_{i,k}^h-1)\right]
\end{aligned}
\end{equation}
According to Eq.~\eqref{eq:progress_gap}, the average top model update index satisfies:
\begin{equation}\label{eq:ca_gap}
\frac1{\tau_i^h}\sum_{k=1}^{\tau_i^h}(r_{i,k}^h-1)
 =\zeta_i^h+\frac{\tau_i^h-1}{2}
\end{equation}
Combining Eqs.~\eqref{eq:ca_topdrift}--\eqref{eq:ca_gap}, we obtain the following expected gradient deviation bounds:
\begin{equation}\label{eq:ca_biasproof}
\begin{aligned}
\|\mathbb E_h[\overline{\boldsymbol g}_b^h]
 -\nabla F_b(\boldsymbol w^h)\|&\le B_b^h\\
\|\mathbb E_h[\overline{\boldsymbol g}_p^h]
 -\nabla F_p(\boldsymbol w^h)\|&\le B_p^h
\end{aligned}
\end{equation}
Applying the standard inner-product inequality
$2\langle\boldsymbol x,\boldsymbol y\rangle
\ge\|\boldsymbol x\|^2-\|\boldsymbol x-\boldsymbol y\|^2$
to Eq.~\eqref{eq:ca_biasproof} establishes Eq.~\eqref{eq:ca_direction}.
\end{IEEEproof}

\begin{lemma}\label{lem:ca_second}
Under Assumption~\ref{ass:coasfl_moment}, the second moment of the aggregated update satisfies the following upper bound:
\begin{equation}\label{eq:ca_second}
\mathbb E_h\|\boldsymbol d^h\|^2
 \le(\gamma^h)^2\sigma_b^2+(R^h)^2\sigma_p^2
\end{equation}
\end{lemma}

\begin{IEEEproof}
According to Eq.~\eqref{eq:ca_actual}, Jensen's inequality, and Assumption~\ref{ass:coasfl_moment}, the second moment of the aggregated update can be bounded as follows:
\begin{equation}\label{eq:ca_secondproof}
\begin{aligned}
\mathbb E_h\|\boldsymbol d^h\|^2
&=(\gamma^h)^2\mathbb E_h\|\overline{\boldsymbol g}_b^h\|^2
 +(R^h)^2\mathbb E_h\|\overline{\boldsymbol g}_p^h\|^2\\
&\le(\gamma^h)^2\sum_i\frac{\rho_i^h}{\tau_i^h}
 \sum_{k=1}^{\tau_i^h}
 \mathbb E_h\|\widetilde{\boldsymbol g}_{b,i}^{h,k}\|^2\\
&\quad+R^h\sum_{r=1}^{R^h}
 \mathbb E_h\|\widetilde{\boldsymbol g}_p^{h,r}\|^2\\
&\le(\gamma^h)^2\sigma_b^2+(R^h)^2\sigma_p^2
\end{aligned}
\end{equation}
\end{IEEEproof}

\begin{theorem}\label{thm:ca_convergence}
If all bottom models share the same initial parameters in each round and $F$ is lower bounded by $F_{\inf}$, then under Assumptions~\ref{ass:coasfl_smooth}--\ref{ass:coasfl_moment}, the mean-square gradients after $H$ aggregation rounds satisfy the following convergence bound:
\begin{equation}\label{eq:ca_theorem}
\begin{aligned}
&\frac1H\sum_{h=0}^{H-1}\mathbb E\Big[
 \gamma^h\|\nabla F_b(\boldsymbol w^h)\|^2
 +R^h\|\nabla F_p(\boldsymbol w^h)\|^2\Big]\\
&\quad\le\frac{2[F(\boldsymbol w^0)-F_{\inf}]}{\eta H}\\
&\qquad+\frac1H\sum_{h=0}^{H-1}\mathbb E\Big[
 \gamma^h(B_b^h)^2+R^h(B_p^h)^2\Big]\\
&\qquad+\frac{L\eta}{H}\sum_{h=0}^{H-1}\mathbb E\Big[
 (\gamma^h)^2\sigma_b^2+(R^h)^2\sigma_p^2\Big]
\end{aligned}
\end{equation}
\end{theorem}

\begin{IEEEproof}
According to Assumption~\ref{ass:coasfl_smooth} and Lemmas~\ref{lem:ca_direction}--\ref{lem:ca_second}, the smoothness descent inequality~\cite{gower2019sgd} gives the following bound on the expected loss decrease:
\begin{equation}\label{eq:ca_descent}
\begin{aligned}
&\mathbb E_h[F(\boldsymbol w^{h+1})-F(\boldsymbol w^h)]\\
&\quad\le-\eta\mathbb E_h\langle\nabla F(\boldsymbol w^h),
 \boldsymbol d^h\rangle
 +\frac{L\eta^2}{2}\mathbb E_h\|\boldsymbol d^h\|^2\\
&\quad\le-\frac\eta2\Big[
 \gamma^h\|\nabla F_b(\boldsymbol w^h)\|^2
 +R^h\|\nabla F_p(\boldsymbol w^h)\|^2\Big]\\
&\qquad+\frac\eta2\Big[
 \gamma^h(B_b^h)^2+R^h(B_p^h)^2\Big]\\
&\qquad+\frac{L\eta^2}{2}\Big[
 (\gamma^h)^2\sigma_b^2+(R^h)^2\sigma_p^2\Big]
\end{aligned}
\end{equation}
Taking the total expectation and summing Eq.~\eqref{eq:ca_descent} over $H$ rounds, we obtain:
\begin{equation}\label{eq:ca_telescope}
\begin{aligned}
&\frac\eta2\sum_{h=0}^{H-1}\mathbb E\Big[
 \gamma^h\|\nabla F_b(\boldsymbol w^h)\|^2
 +R^h\|\nabla F_p(\boldsymbol w^h)\|^2\Big]\\
&\quad\le F(\boldsymbol w^0)-F_{\inf}\\
&\qquad+\frac\eta2\sum_{h=0}^{H-1}\mathbb E\Big[
 \gamma^h(B_b^h)^2+R^h(B_p^h)^2\Big]\\
&\qquad+\frac{L\eta^2}{2}\sum_{h=0}^{H-1}\mathbb E\Big[
 (\gamma^h)^2\sigma_b^2+(R^h)^2\sigma_p^2\Big]
\end{aligned}
\end{equation}
Dividing both sides of Eq.~\eqref{eq:ca_telescope} by $\eta H/2$ establishes the convergence bound in Eq.~\eqref{eq:ca_theorem}.
\end{IEEEproof}
According to Theorem~\ref{thm:ca_convergence}, the gradient deviations $B_b^h$ and $B_p^h$ influence the convergence bound.
Let $\boldsymbol\mu_{0,c}^h$ and $\boldsymbol\mu_{S,c}^{h,r}$ denote the expected class-$c$ top model gradients under the balanced client distribution and selected activation set $S^{h,r}$, respectively.
We define the corresponding conditional gradient deviation and reference gradient bound as follows:
\begin{equation}\label{eq:ca_cond}
\begin{aligned}
G_{\rm cls}^h
 &=\max_c\|\boldsymbol\mu_{0,c}^h\|\\
\Delta_{\rm cond}^{h,r}
 &=\sum_{c=1}^{M}\Phi_{S^{h,r}}(c)
 \|\boldsymbol\mu_{S,c}^{h,r}
 -\boldsymbol\mu_{0,c}^h\|
\end{aligned}
\end{equation}
According to Pinsker's inequality, the label-dependent top model gradient deviation satisfies the following bound:
\begin{equation}\label{eq:ca_kl}
\begin{aligned}
&\left\|
\frac1{|S^{h,r}|}\sum_{i\in S^{h,r}}
 \nabla F_{p,i}(\boldsymbol w^h)
 -\nabla F_p(\boldsymbol w^h)
\right\|\\
&\quad\le G_{\rm cls}^h
 \sqrt{2D_{\rm KL}(\Phi_{S^{h,r}}\|\Phi_0)}
 +\Delta_{\rm cond}^{h,r}
\end{aligned}
\end{equation}
Hence, the KL-divergence constraint limits the label-dependent component of the server-side gradient deviation for candidate combinations satisfying Eq.~\eqref{eq:valid_combination_constraints}.
For timeout-triggered fallback combinations, Eq.~\eqref{eq:ca_kl} remains applicable with their actual KL divergence.


For the bottom model, substituting Eqs.~\eqref{eq:aggregation_correction_factor}--\eqref{eq:configured_aggregation_weight} into Eq.~\eqref{eq:ca_weights}, the effective gradient contribution is proportional to
$\rho_i^h\propto \frac{(\tau_i^h)^2}{\tau_i^h+1}\frac{1}{1+\zeta_i^h/R^h}$.
Let $a_i^h$ denote the normalized weights proportional to $(\tau_i^h)^2/(\tau_i^h+1)$, and let $t_i^h=(1+\zeta_i^h/R^h)^{-1}$.
Since $\rho_i^h=a_i^ht_i^h/\sum_j a_j^ht_j^h$, the difference in weighted progress gaps can be expressed as follows:
\begin{equation}\label{eq:ca_gapcontrol}
\begin{aligned}
&\sum_i a_i^h\zeta_i^h
-\sum_i\rho_i^h\zeta_i^h\\
&\quad=
\frac{\sum_{i,j}a_i^ha_j^h
 (\zeta_i^h-\zeta_j^h)^2t_i^ht_j^h}
 {2R^h\sum_j a_j^ht_j^h}
\ge0
\end{aligned}
\end{equation}
According to Eqs.~\eqref{eq:ca_kl} and~\eqref{eq:ca_gapcontrol}, the activation selection and aggregation weights control different sources of gradient deviation.
The former bounds the label-dependent deviation of top model updates, while the latter reduces the weighted progress mismatch of bottom model updates.
Consequently, the convergence bound reflects the combined effects of client contribution imbalance and asynchronous model progress on the overall optimization.

\section{Evaluation}\label{sec:evaluation}


\subsection{Experimental Settings}
\textbf{System Deployment.} 
We conducted a series of comprehensive experiments on an edge computing hardware prototype system. 
For the PS, we employed a deep learning GPU workstation equipped with an Intel(R) Xeon(R) Platinum 8360Y CPU, four NVIDIA A100 80GB PCIe, and 503 GiB of RAM. 
To create a heterogeneous computing environment, we incorporated 80 NVIDIA Jetson devices, consisting of 30 Jetson TX2, 40 Jetson NX, and 10 Jetson AGX units. 
The detailed technical specifications of the Jetson TX2, NX, and AGX models are provided in Table \ref{tab:jetson}. 
The experimental setup utilized a software platform based on Docker Swarm \cite{naik2016building} and the PyTorch deep learning framework\cite{paszke2019pytorch}. 
Docker Swarm, a distributed systems orchestration tool, facilitates the deployment and management of a multi-node architecture, enabling seamless monitoring and coordination of the devices’ operational states. 
PyTorch, a widely adopted deep learning framework, was used for model training across the distributed hardware nodes. 
To optimize inter-device communication, we integrated the Message Passing Interface (MPI) \cite{gabriel2004open}, a high-performance communication protocol that provides a robust set of functions for efficient data transmission and synchronization across the system. 
This combination of hardware and software enabled the evaluation of \method's scalability and performance in a heterogeneous edge computing environment.


\begin{table}[!t]
\caption{Device technical specifications.}
\centering
\resizebox{\linewidth}{!}{
\begin{tabular}{lcc}
\hline
& \textbf{AI Performance} & \textbf{GPU Type}  \\ 
\hline
Jetson TX2 & 6 TOPs & 256-core Pascal \\ 
Jetson NX  & 21 TOPs & 384-core Volta\\ 
Jetson AGX  & 32 TOPs & 512-core Volta\\ \hline \hline
& \textbf{CPU Type} & \textbf{RAM} \\  \hline
Jetson TX2 & Denver 2 and ARM 4 & 8 GB LPDDR4\\ 
Jetson NX & 6-core Carmel ARM 8 & 8 GB LPDDR4x\\ 
Jetson AGX  & 8-core Carmel ARM 8 & 32 GB LPDDR4x \\  \hline \hline
& \textbf{CPU Frequency} & \textbf{GPU Frequency} \\  \hline
Jetson TX2 & 2.0GHz & 1.12GHz\\ \hline
Jetson NX & 1.9GHz & 1.1GHz\\ \hline
Jetson AGX & 2.2GHz & 1.37GHz\\ \hline
\hline
\end{tabular}
}
\label{tab:jetson}
\vspace{-3mm}
\end{table}

\textbf{Setting of System Heterogeneity.} 
The following experimental setup is employed to enable clients with heterogeneous computational and communication capabilities.

1) \textbf{\textit{Computation Setting.}} 
The Jetson TX2, NX, and AGX devices are configurable to operate in distinct modes (4 for TX2, 8 for NX/AGX), each specifying a particular number of active CPU cores and CPU/GPU frequencies. 
These modes enable devices to exhibit different computational performance.
For instance, the AGX in its highest-performance mode (Mode 0) can achieve training speeds up to 100$\times$ faster than the TX2 in its lowest-performance mode (Mode 1). 
Furthermore, to simulate time-varying resource availability, the device modes are randomly altered every 10 aggregation rounds.

2) \textbf{\textit{Communication Setting.}} 
Devices are connected to the PS via Wi-Fi routers, with a total of 80 devices grouped into four groups of 20 devices each. 
These groups are positioned at varying distances from the Wi-Fi routers, specifically at 2m, 8m, 14m, and 20m. 
Due to random channel interference and contention between devices, the bandwidth between the PS and the devices fluctuates during training, ranging from 1 Mbps to 30 Mbps, as measured by iperf3.

\textbf{Applications and Models.} 
To evaluate the performance of the \method in real-world scenarios, we conduct experiments on three representative applications using pre-trained transformer models of diverse scales:




1) \textbf{\textit{Object Recognition.}} 
CIFAR-100 \cite{krizhevsky2009learning} is a widely used benchmark dataset for object recognition, which consists of 60,000 32$\times$32 colour images (50,000 for training and 10,000 for test) across 100 fine-grained classes. 
We utilize a well-known vision transformer model, Deit-Base \cite{touvron2022deit}, which has a size of 475MB and comprises 12 transformer layers that process images through 16$\times$16 patch embeddings. 

2) \textbf{\textit{Natural Language Processing.}}
We utilize the Question-answering Natural Language Inference (QNLI) dataset \cite{wang2018glue}, a classification dataset consisting of question-sentence pairs, in this application.
There are 104,473 samples for training and 5,463 samples for testing in the QNLI dataset.
We train a RoBERTa-Base\cite{liu2019roberta}  model with a size of 501MB on the QNLI dataset, which is composed of 12 transformer layers.

3) \textbf{\textit{Text Classification.}}
We adopt the 20 Newsgroups dataset \cite{lang1995newsweeder}, a classical text classification benchmark containing 18,846 documents (11,314 for training and 7,532 for testing) across 20 categories. 
For this task, we employ the Qwen2.5-0.5B \cite{hui2024qwen2}, a size of 1,884MB LLM with 24 transformer layers.

\textbf{Setting of Statistical Heterogeneity.} In the experiments, training samples of each client are drawn from a Dirichlet distribution\cite{yurochkin2019bayesian}, \ie, $v\sim Dir(\delta\psi)$, where $\psi$ characterizes a prior class distribution, and $\delta>0$ is a concentration parameter controlling the identicalness among clients. With $\delta\to\infty$, all clients have identical distributions to the prior class distribution (\ie, IID); with $\delta\to 0$ each client holds data samples from only one class, which indicates high degree of statistical heterogeneity. We specify 6 values (\eg,$\infty$,1,0.5,0.25,0.2,0.1) for $\delta$ to generate different data distributions that cover a spectrum of identicalness, and define $p=\frac{1}{\delta}$(\eg, $p$=0,1,2,4,5,10) to quantify the non-IID levels. The degree of statistical heterogeneity increases as $p=0$ is a special case of IIDness.

\textbf{Baselines.} 
We measure the effectiveness of \method through a comparison with four baselines.
All the asynchronous baselines adopt an intra-round asynchronous and inter-round synchronous training strategy, which is the same as \method.

1) \textbf{PipeSFL} \cite{gao2024pipesfl} addresses the challenge of system heterogeneity in computation and communication during typical SFL by proposing a novel pipeline asynchronous framework that accelerates SFL execution across heterogeneous devices.

2) \textbf{GAS} \cite{yang2025gas} represents the state-of-the-art in asynchronous SFL approaches, which addresses statistical heterogeneity challenges through an innovative approach of sampling and generating intermediate activations.

3) \textbf{SplitFed} \cite{thapa2022splitfed} is the first SFL approach, which combines FL with SL by utilizing parallel client-side computing to accelerate model training in a synchronous setting.

\begin{figure*}[t]
    \centering
    \subfloat[CIFAR-100\label{fig:time_acc_iid_deit_cifar}]{
        \includegraphics[height=4.2cm]{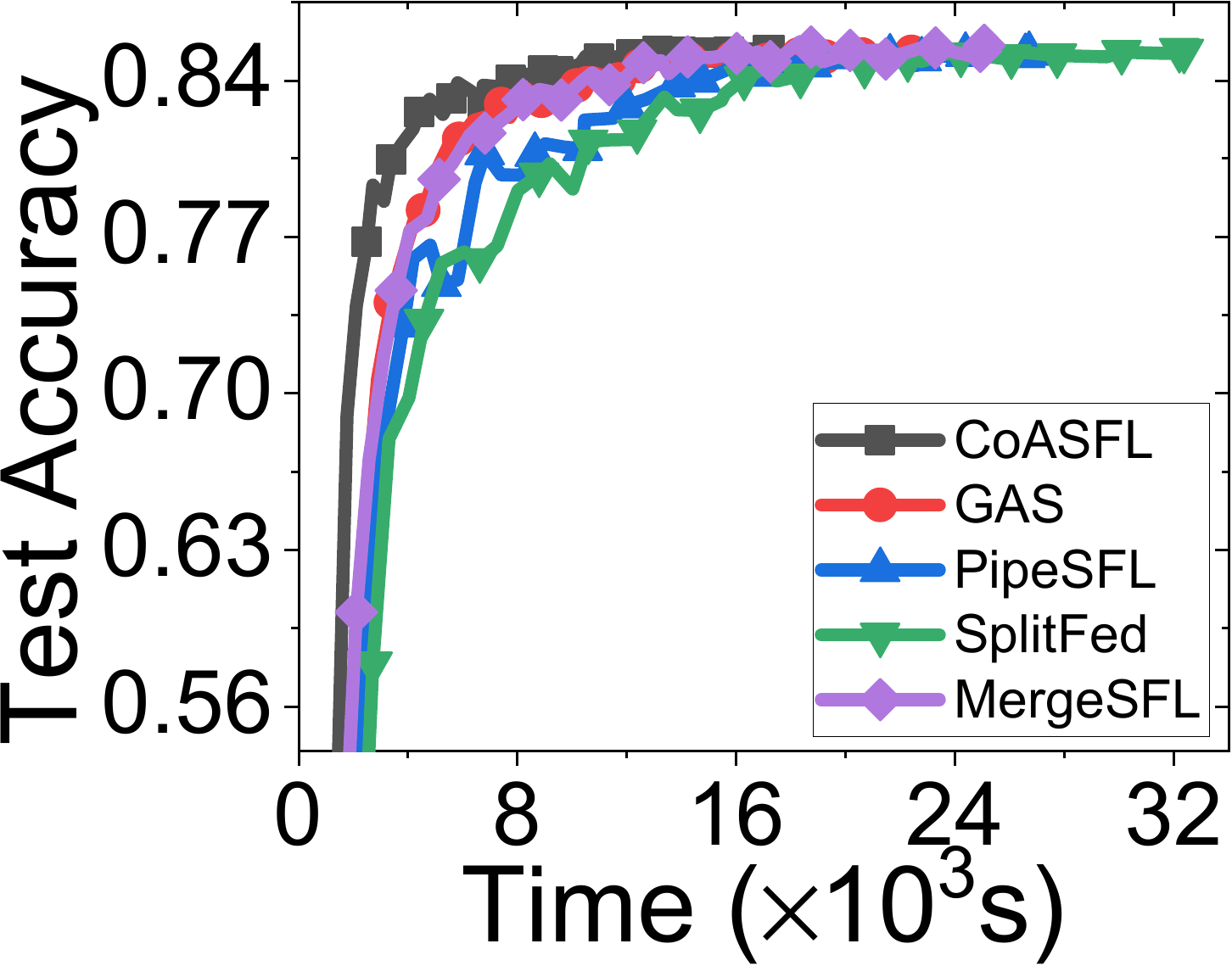}
    }
    \hfill
    \subfloat[QNLI\label{fig:time_acc_iid_roberta_qnli}]{
        \includegraphics[height=4.2cm]{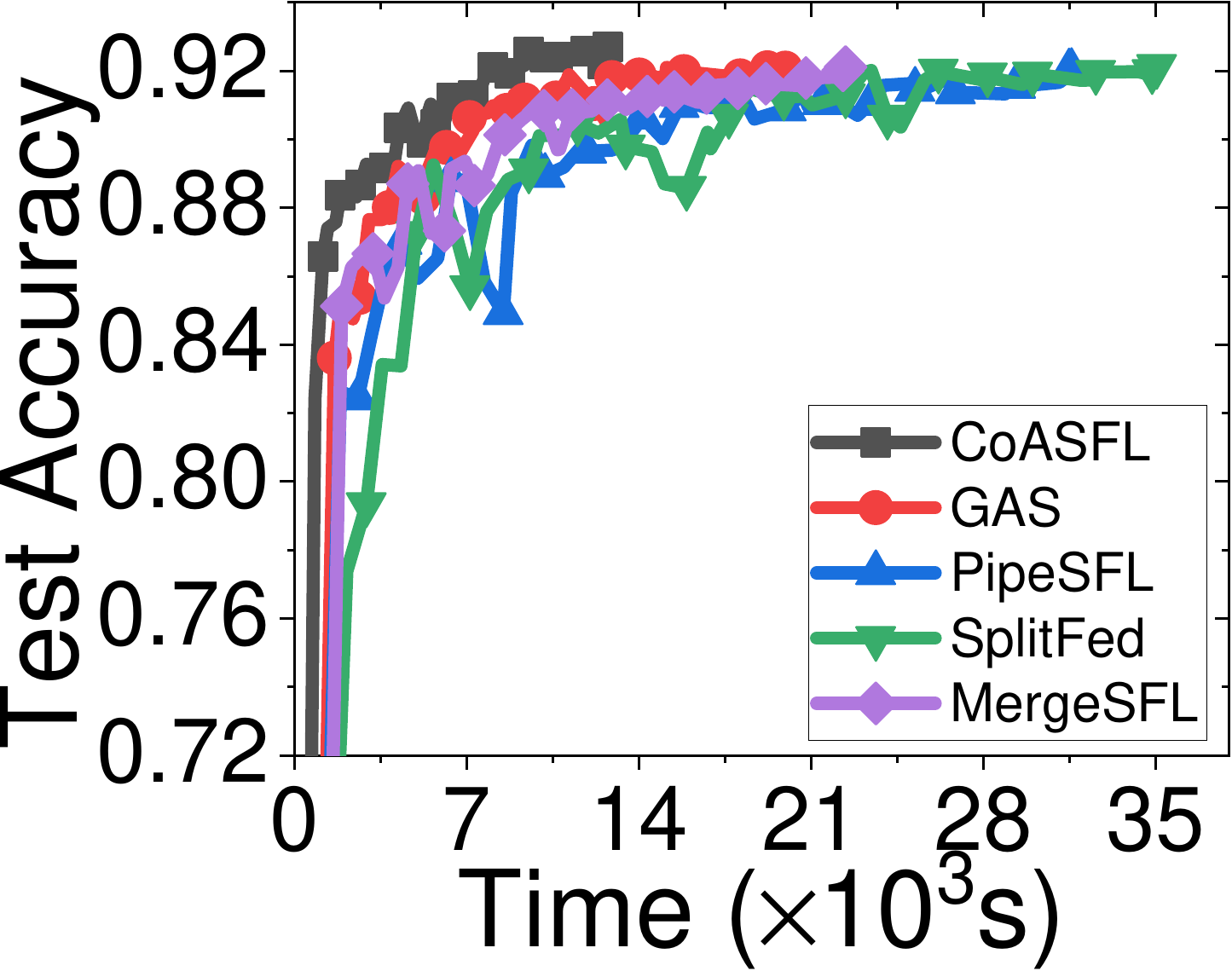}
    }
    \hfill
    \subfloat[Newsgroups\label{fig:time_acc_iid_qwen_newsgroup}]{
        \includegraphics[height=4.2cm]{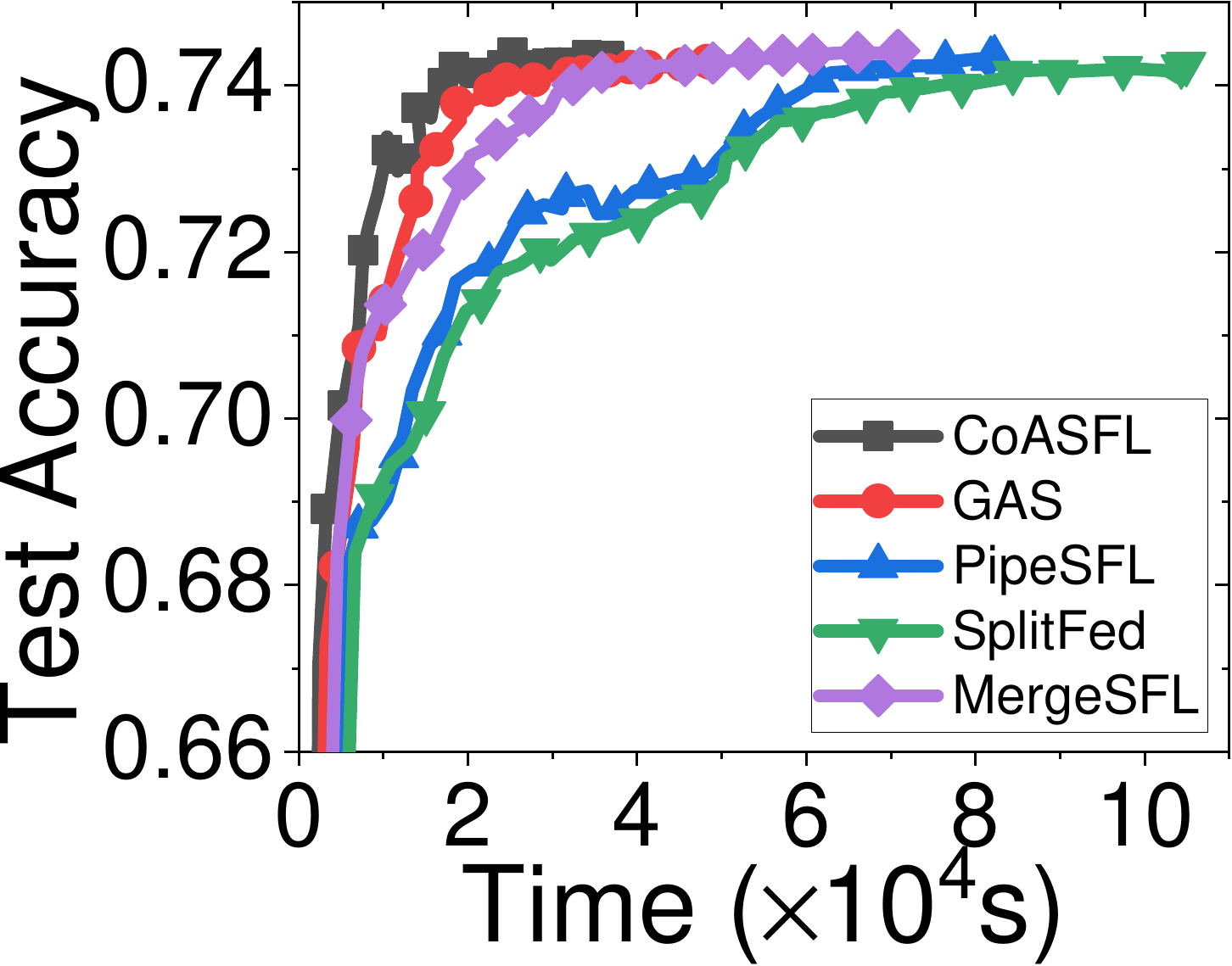}
    }
    \caption{Test accuracy on the three IID datasets}
    \label{fig:time_acc_iid}
\end{figure*}


\begin{figure*}[t]
    \centering
    \subfloat[CIFAR-100\label{fig:time_acc_noniid_deit_cifar}]{
        \includegraphics[height=4.2cm]{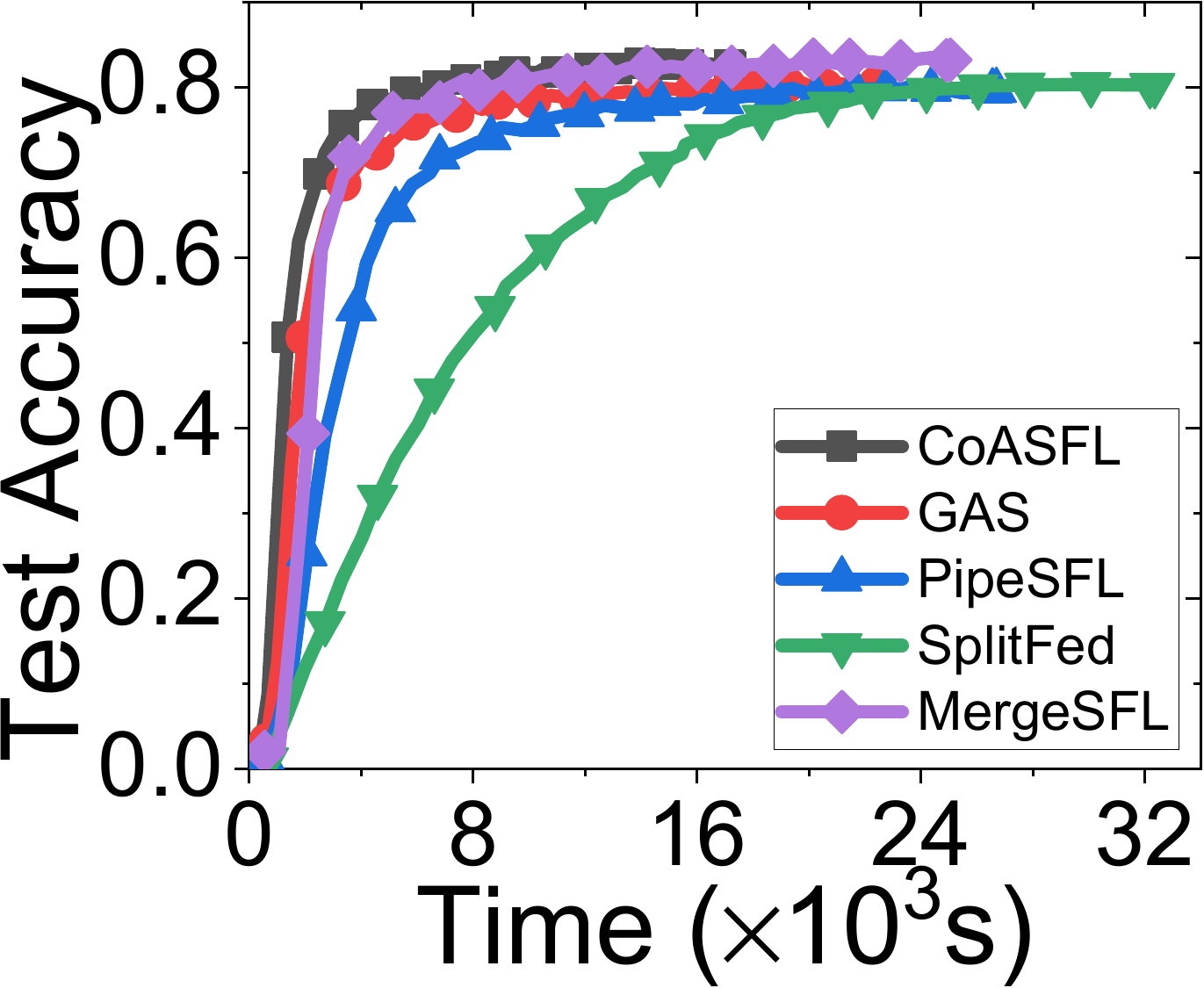}
    }
    \hfill
    \subfloat[QNLI\label{fig:time_acc_noniid_roberta_qnli}]{
        \includegraphics[height=4.2cm]{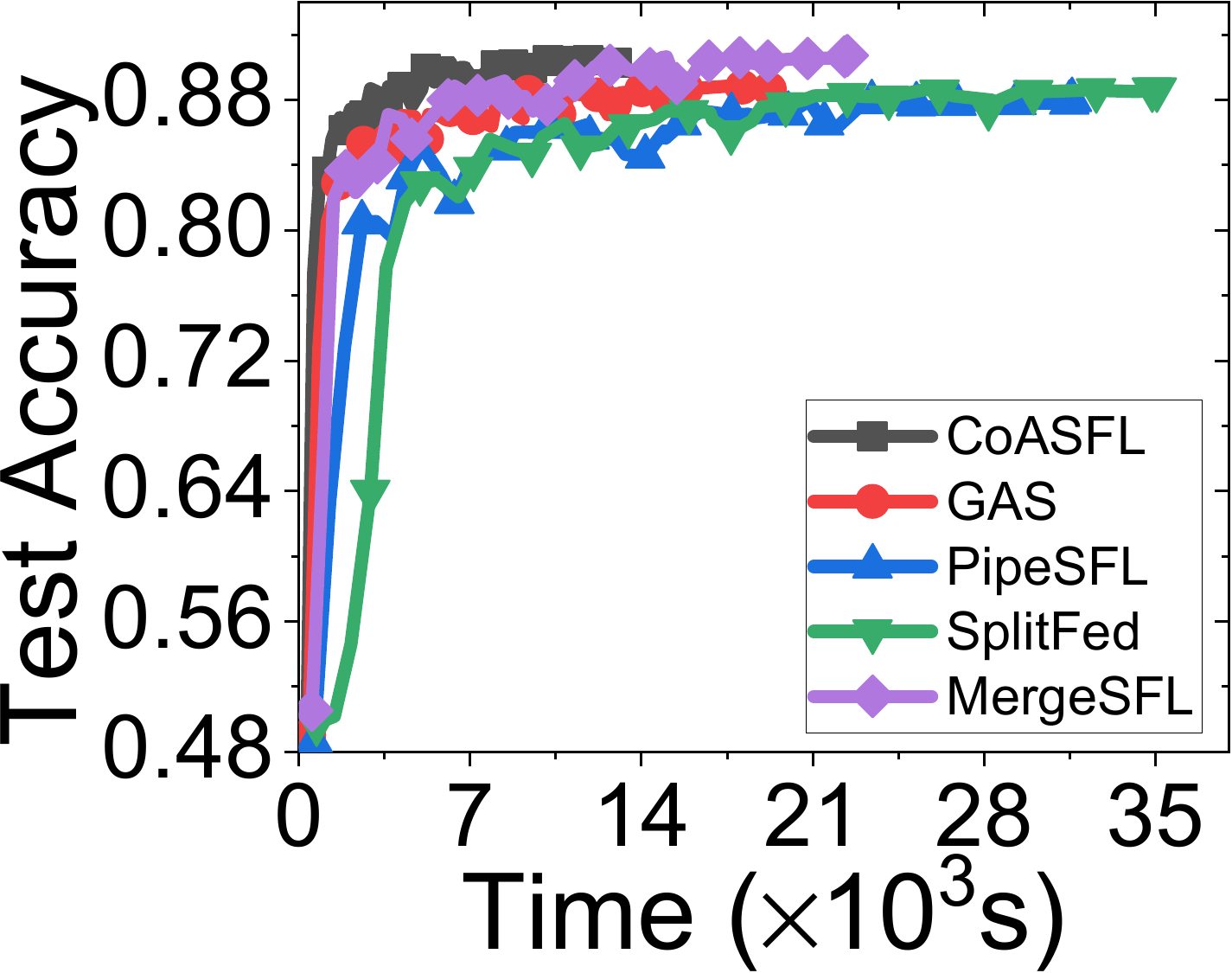}
    }
    \hfill
    \subfloat[Newsgroups\label{fig:time_acc_noniid_qwen_newsgroup}]{
        \includegraphics[height=4.2cm]{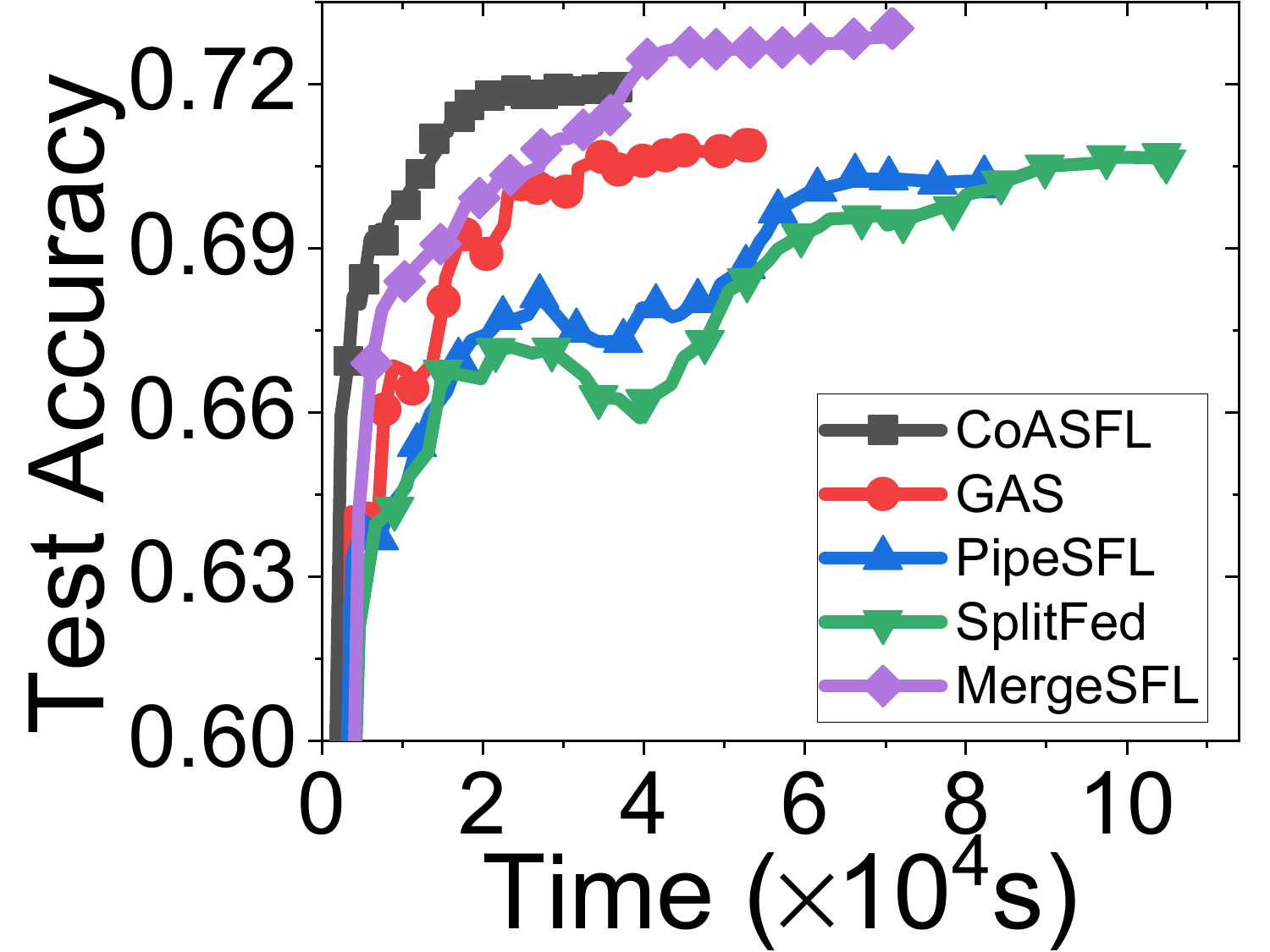}
    }
    \caption{Test accuracy on the three non-IID datasets}
    \label{fig:time_acc_noniid}
\end{figure*}


4) \textbf{MergeSFL} \cite{10597905} is the state-of-the-art approach in synchronous SFL,  introducing the activation merging mechanism with client selection to address system and statistical heterogeneity simultaneously.

\textbf{Metrics.} 
We adopt the following evaluation metrics:

1) \textbf{Test Accuracy} measures the global model's predictive performance on the test set across training rounds, reporting the final converged accuracy.

2) \textbf{Time-to-Accuracy} measures the wall-clock time to reach a target accuracy achievable by all approaches. Clients' average waiting time is also recorded to reflect training efficiency. 

3) \textbf{Network Traffic} is calculated by summing the traffic for transmitting models or activations between the PS and clients when achieving a target accuracy.


\textbf{Experimental Parameters.}
All experiments are run for 100 rounds with an average local update frequency of 20.
For DeiT on CIFAR-100, we adopt a learning rate of 0.1, a decay of 0.993~\cite{gower2019sgd,shi2024heterogeneous}, and a batch size of 64.
Conversely, for RoBERTa on QNLI and Qwen2.5 on Newsgroups, we employ a learning rate of 5e-5, a linear scheduler with 10\% warm-up~\cite{loshchilov2017decoupled}, a batch size of 16, and a sequence length of 128.
For the SFL approaches, DeiT and RoBERTa are split at the 6th layer, while Qwen2.5 is split at the 3rd layer.


\subsection{Overall Performance}

Firstly, we conduct a set of experiments on the IID datasets to evaluate the convergence performance of the \method and the baselines.
The overall training processes of these approaches are presented in Fig. \ref{fig:time_acc_iid}.
According to the results, all the approaches achieve a similar test accuracy on the three datasets eventually.
However, the \method achieves the fastest convergence, outperforming the other approaches on all three datasets.
For instance, as illustrated in Fig. \ref{fig:time_acc_iid_deit_cifar}, \method takes only 10,493s to achieve 85\% accuracy for DeiT on CIFAR-100, while GAS, PipeSFL, SplitFed, and MergeSFL consume 15,852s, 17,774s, 21,369s, and 14,210s, respectively.
Similarly, as shown in Fig. \ref{fig:time_acc_iid_roberta_qnli}, for RoBERTa on QNLI, \method can separately speed up training by about 1.61$\times$, 2.81$\times$, 3.74$\times$, and 2.21$\times$, compared to GAS, PipeSFL, SplitFed, and MergeSFL, respectively.
Moreover, for Qwen2.5 on Newsgroups, as shown in Fig. \ref{fig:time_acc_iid_qwen_newsgroup}, \method reduces the total training time by about 17.6\%, 55.9\%, 65.4\%, and 49.1\%, compared to the four baselines.

\begin{figure*}[t]
    \centering
    \subfloat[CIFAR-100\label{fig:deit_cifar_acc_traffic_iid}]{
        \includegraphics[height=4.2cm]{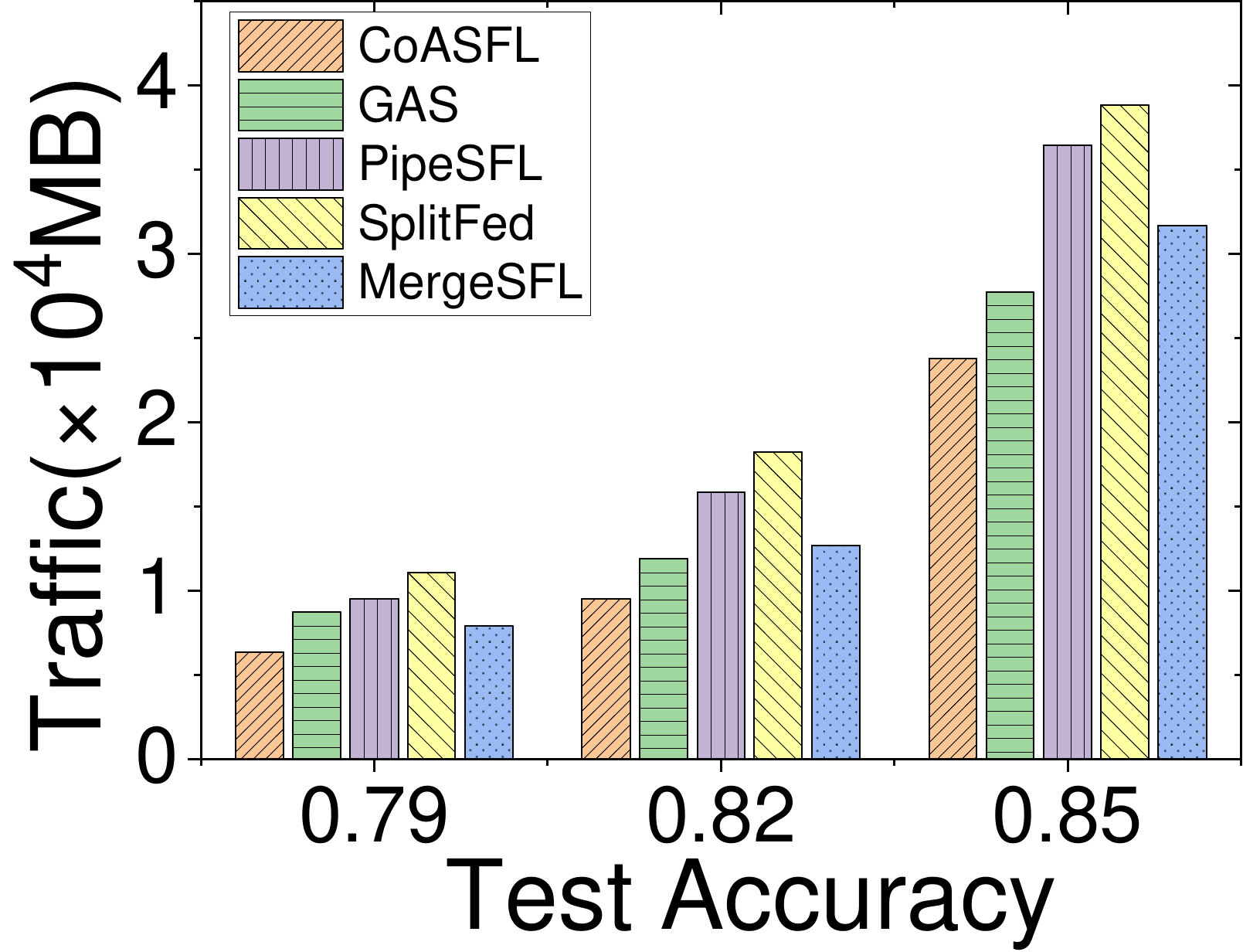}
    }
    \hfill
    \subfloat[QNLI\label{fig:roberta_qnli_acc_traffic_iid}]{
        \includegraphics[height=4.2cm]{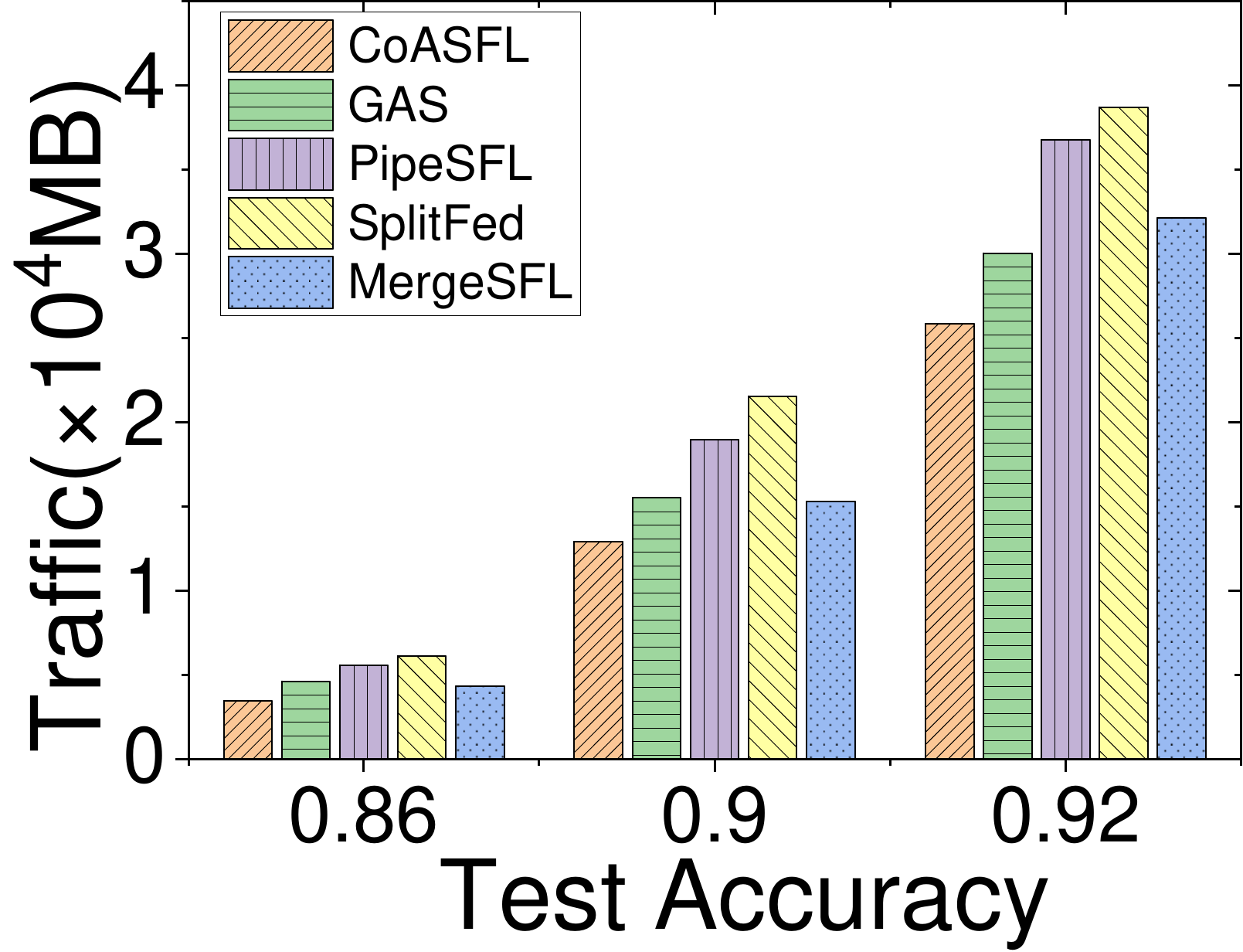}
    }
    \hfill
    \subfloat[Newsgroups\label{fig:qwen_newsgroup_acc_traffic_iid}]{
        \includegraphics[height=4.2cm]{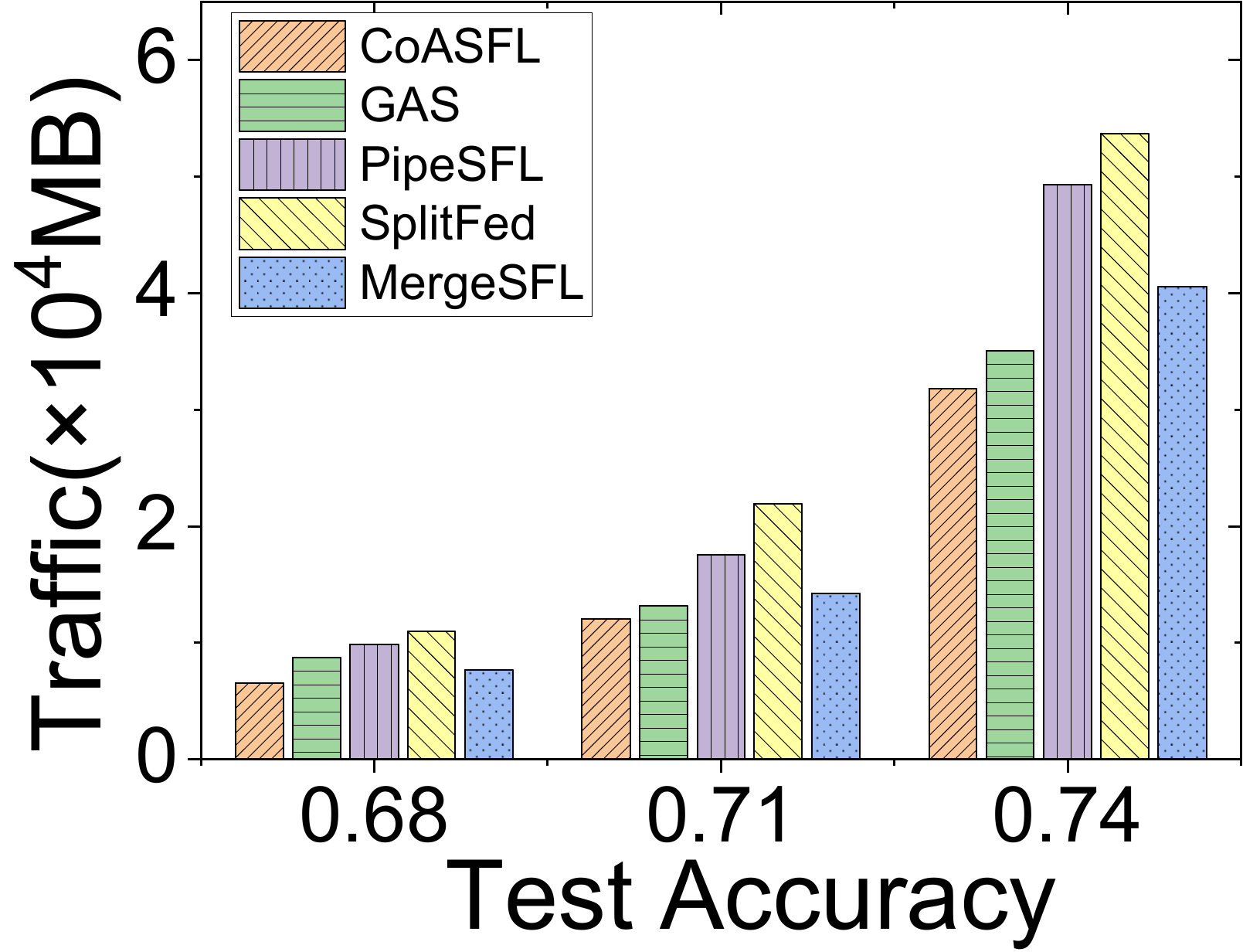}
    }
    \caption{Traffic consumption to achieve target accuracies}
    \label{fig:acc_traffic_iid}
    \vspace{-1mm}

\end{figure*}

\begin{figure*}[t]
    \centering
    \subfloat[CIFAR-100\label{fig:waiting_time_cifar}]{
        \includegraphics[height=4.2cm]{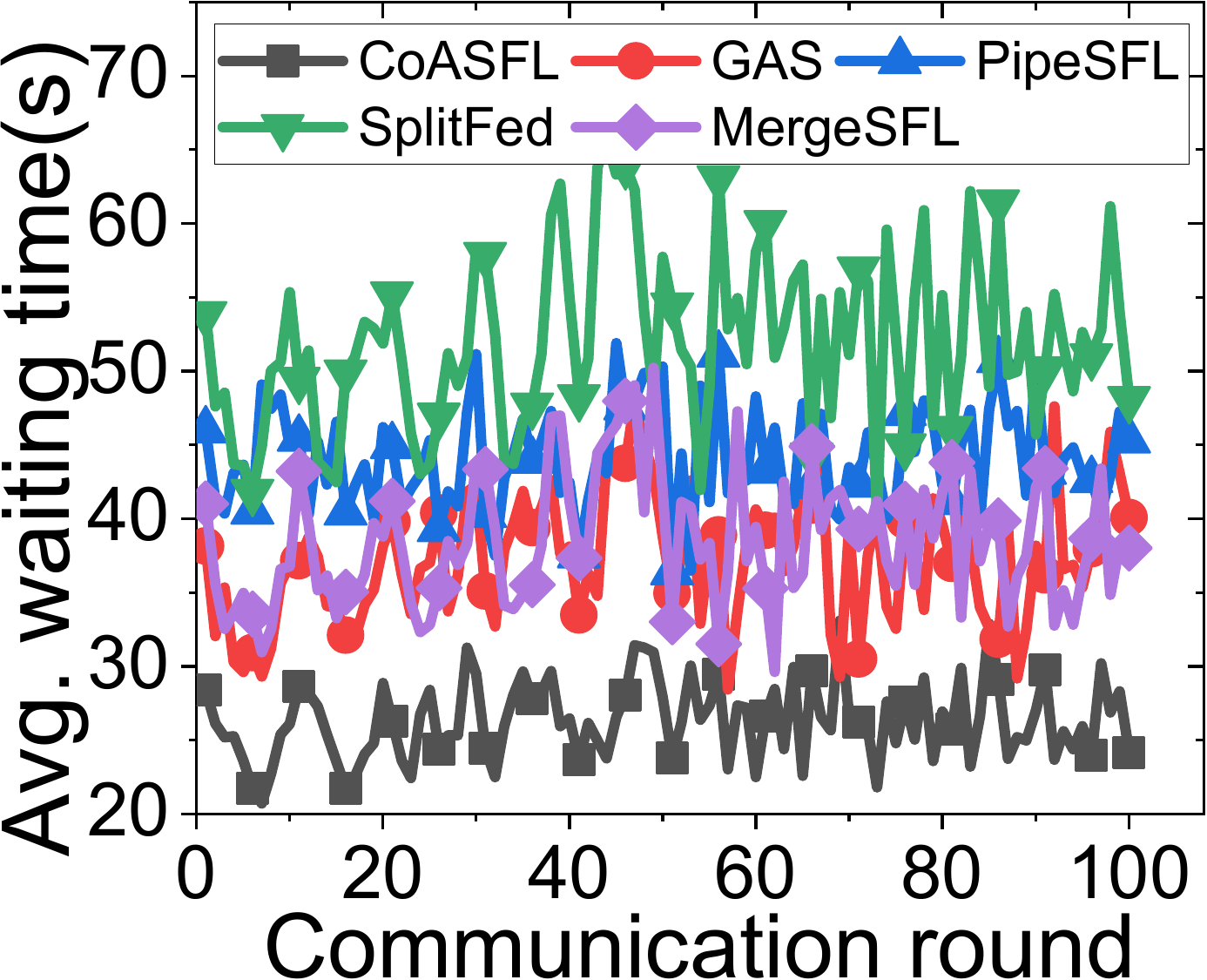}
    }
    \hfill
    \subfloat[QNLI\label{fig:waiting_time_qnli}]{
        \includegraphics[height=4.2cm]{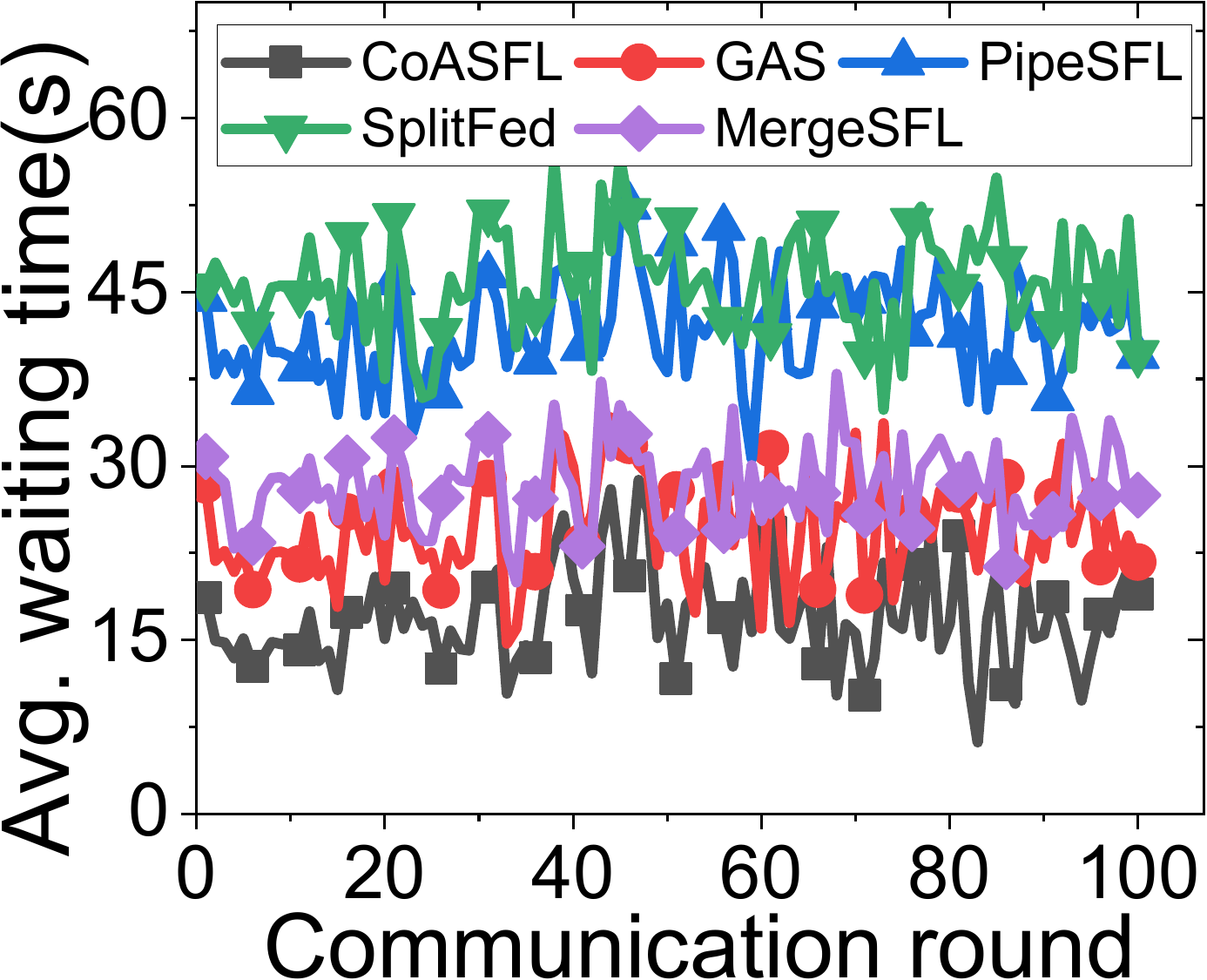}
    }
    \hfill
    \subfloat[Newsgroups\label{fig:waiting_time_newsgroups}]{
        \includegraphics[height=4.2cm]{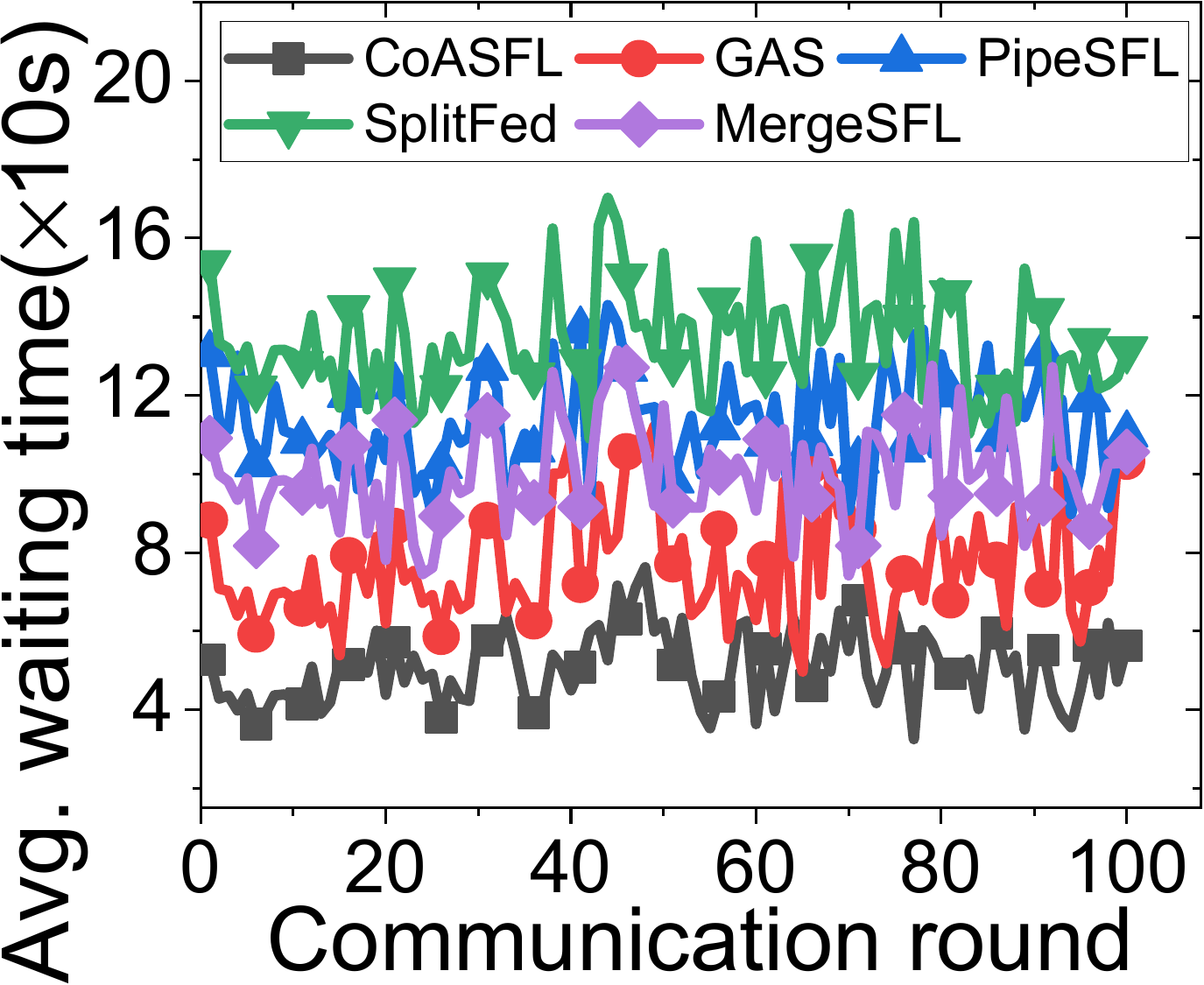}
    }
    \caption{Average waiting time on the three datasets}
    \label{fig:waiting_time}
    \vspace{-1mm}
\end{figure*}


Secondly, we also conduct a set of experiments of these approaches on the three datasets with non-IID level $p$=10, and the results are presented in Fig. \ref{fig:time_acc_noniid}. 
We observe that all the approaches suffer from varying degrees of performance degradation, especially the SplitFed.
However, \method maintains a similar convergence rate as in the IID scenario, and achieves the best performance among these approaches.
For instance, by Fig. \ref{fig:time_acc_noniid_deit_cifar}, \method achieves 82.36\% accuracy in 12,553s for DeiT on CIFAR-100, while GAS, PipeSFL, SplitFed, and MergeSFL take 18,283s, 23,388s, 29,518s, and 24,849s to reach the accuracy of 80.3\%, 79.1\%, 80.6\%, and 83.51\%, respectively.
Similarly, as shown in Fig. \ref{fig:time_acc_noniid_roberta_qnli}, for RoBERTa on QNLI with the same training time of 10,187s, \method improves the test accuracy by about 2.04\%, 4.68\%, 4.96\%, and 1.55\%, compared to GAS, PipeSFL, SplitFed, and MergeSFL, respectively.
In addition, \method with dynamic activation merging and asynchronous training tackles both statistical heterogeneity and client drift better than all the baselines.
Specifically, Fig. \ref{fig:time_acc_noniid_qwen_newsgroup} illustrates that \method separately improves the final accuracy by about 1.18\%,  1.62\%, and 1.31\% compared to the GAS, PipeSFL, and SplitFed, respectively.
Furthermore, \method speeds up training by about 1.96$\times$ compared to MergeSFL.
Besides, \method takes 13,262s to achieve 70\% accuracy for Qwen2.5 on Newsgroups, reducing the training time by about 50.6\%, 79.8\%, 81.1\%, and 67.3\% compared to GAS, PipeSFL, SplitFed, and MergeSFL, respectively.
These results demonstrate that \method is effective in simultaneously tackling the system and statistical heterogeneity.

Furthermore, to illustrate the advantage of the \method in saving communication resources, we present the network traffic consumption of these approaches when achieving different target accuracies in Fig. \ref{fig:acc_traffic_iid}.
According to the results, the network traffic consumption of all approaches increases with the target accuracy for all three datasets.
Furthermore, \method always consumes the least network traffic among all approaches.
Specifically, as shown in Fig. \ref{fig:deit_cifar_acc_traffic_iid}, when achieving 85\% accuracy, \method, GAS, MergeSFL consume 23,774MB, 27,736MB, and 31,699MB respectively, while PipeSFL, SplitFed consume 36,541MB and 38,831MB for DeiT on CIFAR-100.
By Fig. \ref{fig:roberta_qnli_acc_traffic_iid}, for training RoBERTa on QNLI to achieve 92\% accuracy, \method reduces the network traffic consumption by about 4,114MB, 11,066MB, 12,996MB, and 6,219MB, compared to GAS, PipeSFL, SplitFed, and MergeSFL, respectively.
Besides, as illustrated in Fig. \ref{fig:qwen_newsgroup_acc_traffic_iid}, \method saves network traffic consumption by about 9.37\%, 35.55\%, 40.81\%, and 21.62\% when achieving 74\% accuracy for Qwen2.5 on Newsgroups.

\begin{figure*}[t]
    \centering
    \subfloat[CIFAR-100\label{fig:noniid_level_acc_deit_cifar}]{
        \includegraphics[height=4.2cm]{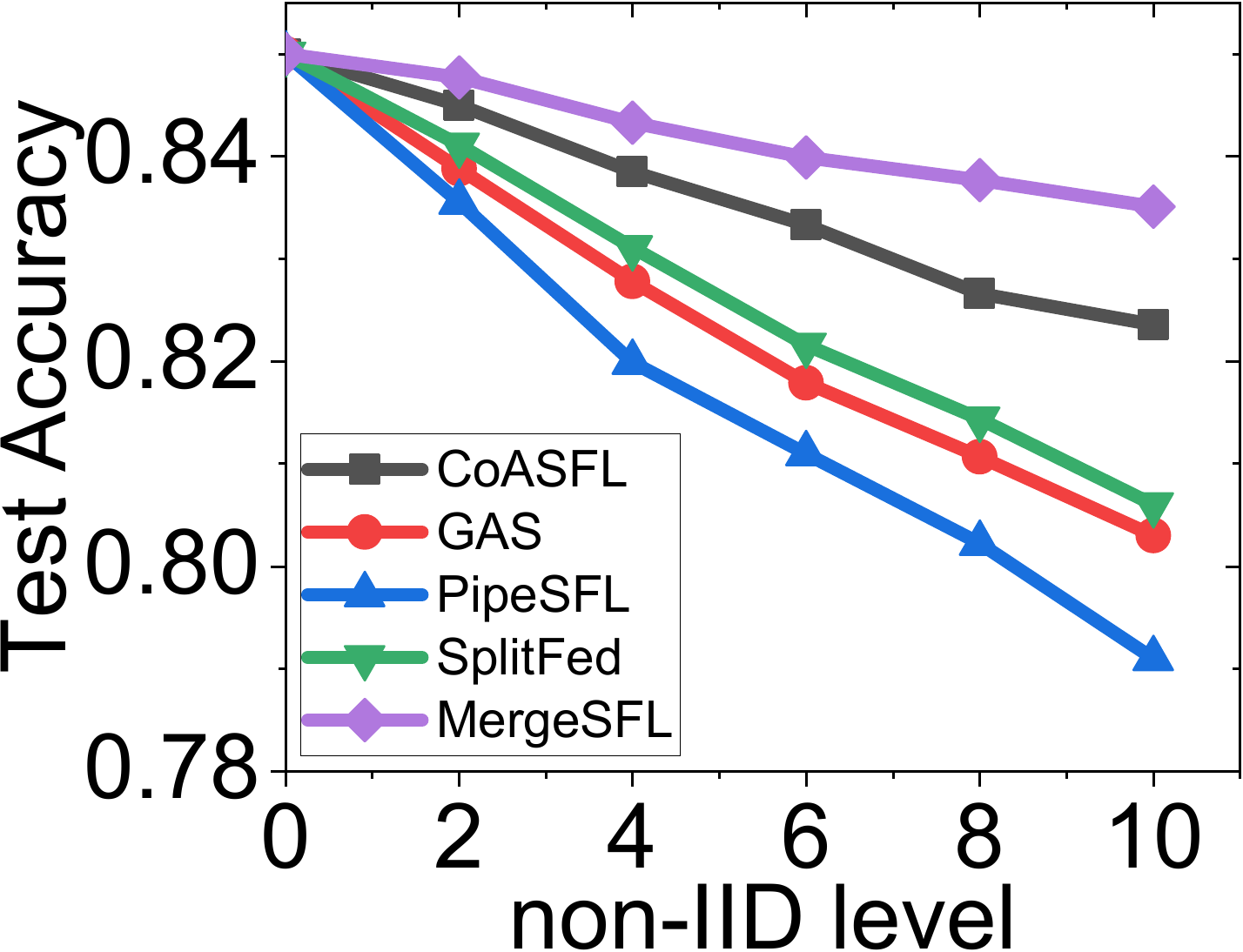}
    }
    \hfill
    \subfloat[QNLI\label{fig:noniid_level_acc_roberta_qnli}]{
        \includegraphics[height=4.2cm]{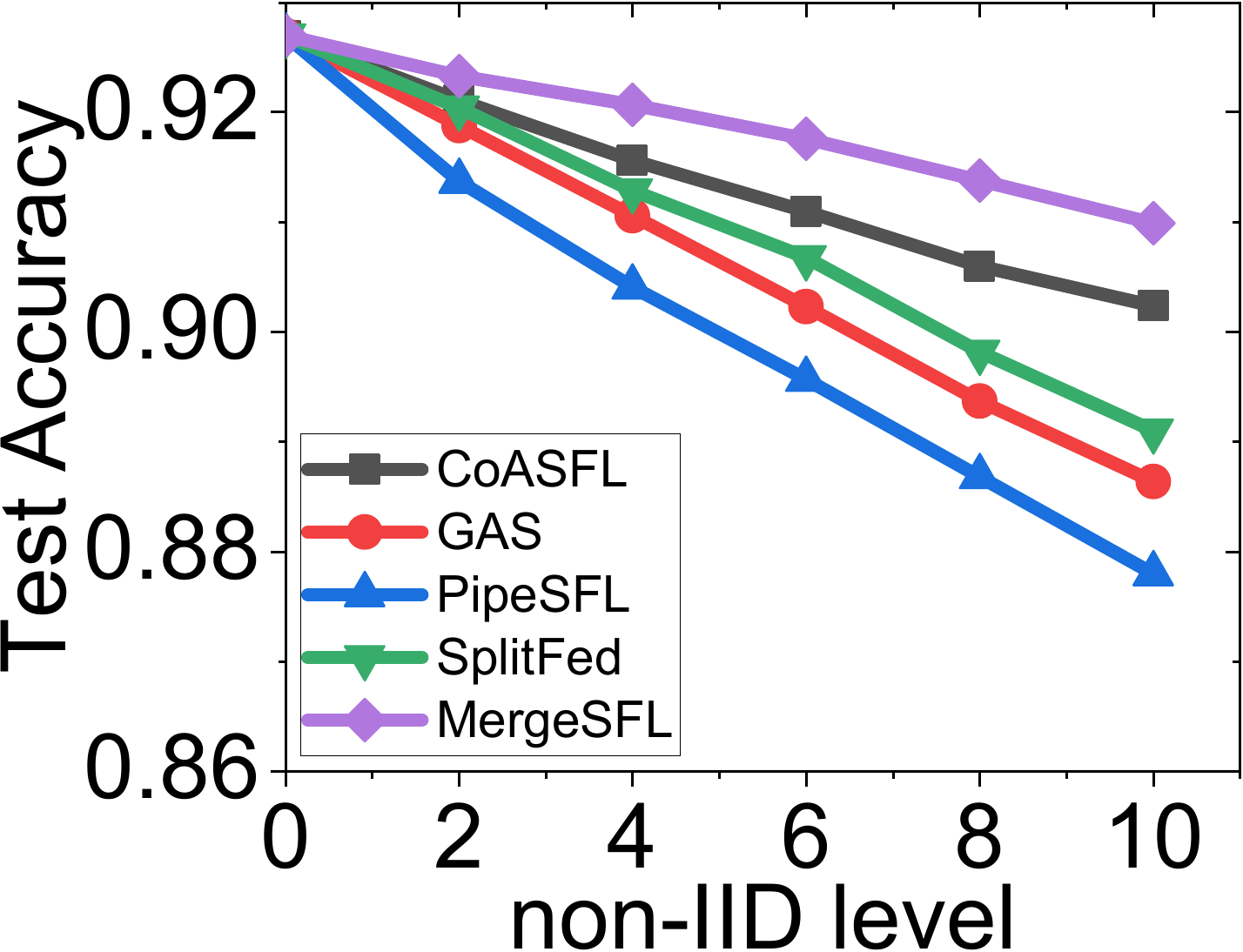}
    }
    \hfill
    \subfloat[Newsgroups\label{fig:noniid_level_acc_qwen_newsgroup}]{
        \includegraphics[height=4.2cm]{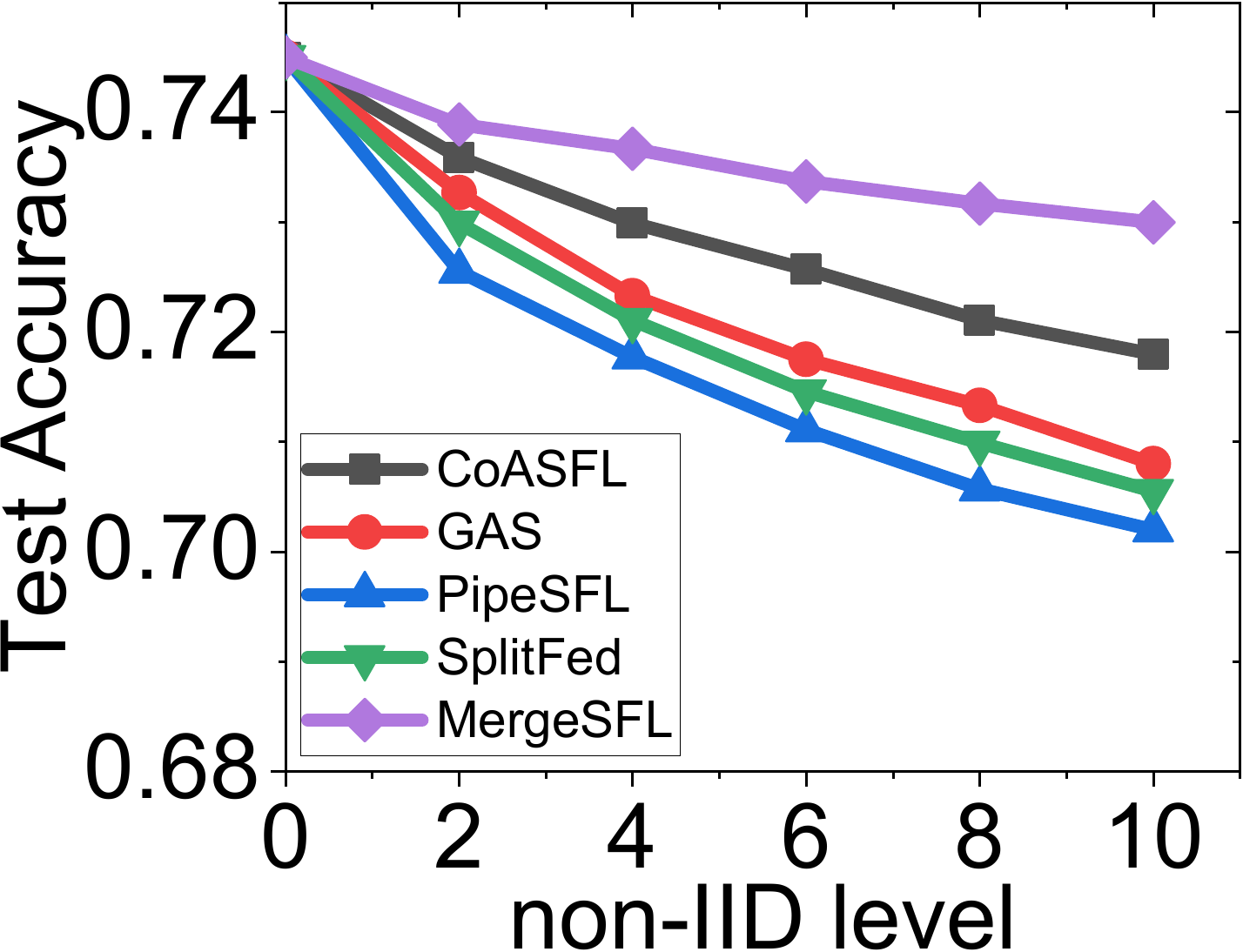}
    }
    \caption{Test accuracy under different non-IID levels}
    \label{fig:noniid_level_acc}
\end{figure*}

\begin{figure}[t]
    \centering
    \subfloat[IID\label{fig:key_strategy_deit_cifar_iid}]{\includegraphics[height=3.3cm]{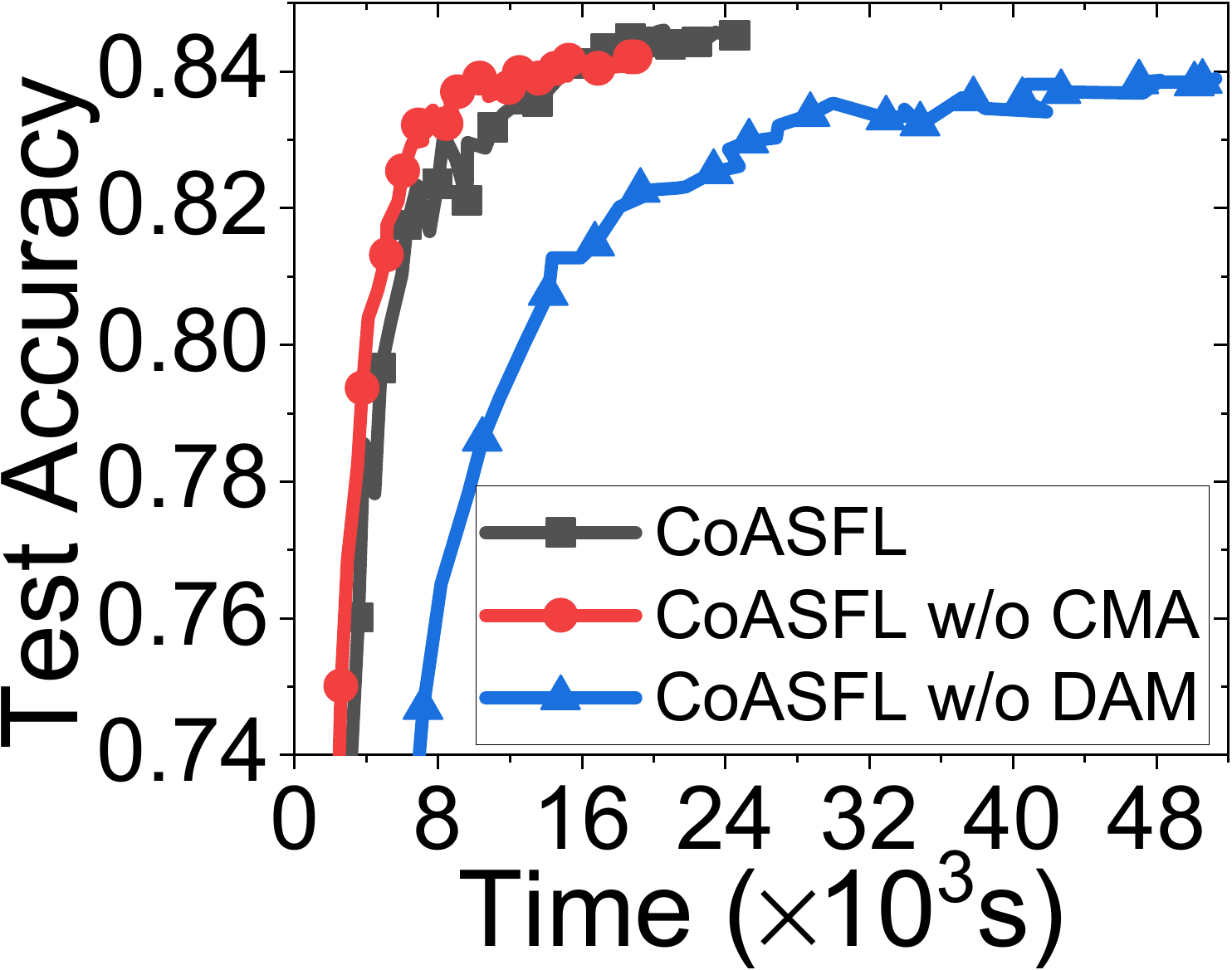}}
    \hspace{0.2cm}
    \subfloat[Non-IID\label{fig:key_strategy_deit_cifar_noniid}]{\includegraphics[height=3.3cm]{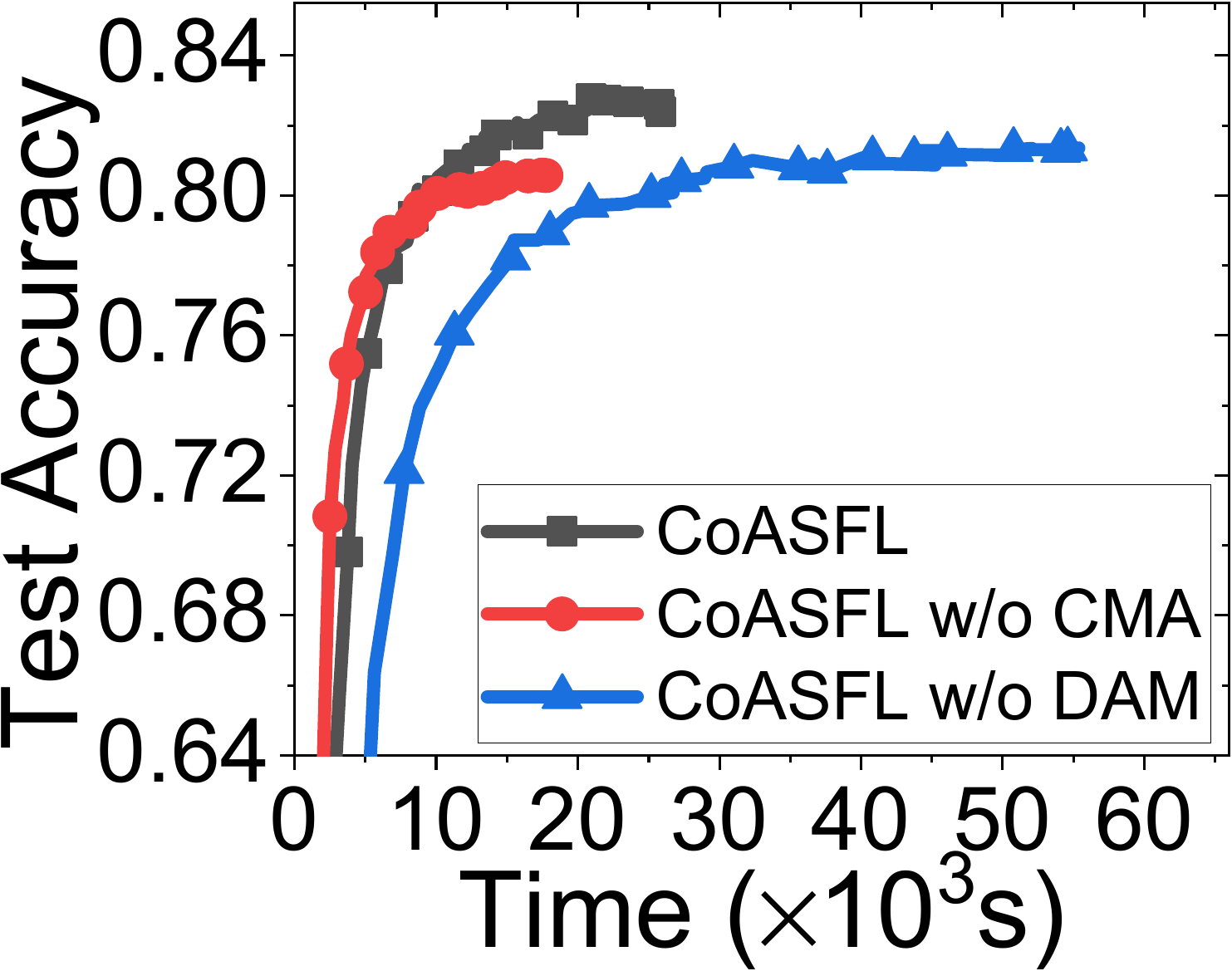}}
    \caption{Effects of DAM and CMA}
    \label{fig:key_strategies}
\end{figure}

 

Finally, to demonstrate the robustness of \method towards system heterogeneity, Fig. \ref{fig:waiting_time} presents the average waiting time of \method and several baseline methods across three datasets.
As shown in Fig. \ref{fig:waiting_time}, \method consistently achieves significantly lower waiting time than GAS and MergeSFL.
Specifically, while GAS utilizes a buffering mechanism, it strictly waits for the buffer to populate without a flexible timeout, leaving it vulnerable to blocking delays caused by stragglers. 
Similarly, MergeSFL relies on an edge-node selection strategy that imposes a rigid synchronization barrier, which inevitably leaves unselected nodes idle and incurs substantial waiting overhead.
Consequently, despite outperforming schemes such as PipeSFL and SplitFed, both baselines exhibit significantly higher waiting times than \method.
For example, as illustrated in Fig. \ref{fig:waiting_time_cifar}, when training DeiT on CIFAR-100, the average waiting time of \method is only 24s, whereas GAS and MergeSFL require 33s and 39s, respectively.
More importantly, \method enables flexible feature matching through dynamic feature aggregation and a timeout mechanism, allowing it to fully utilize system resources and prevent edge nodes from waiting indefinitely. 
This design significantly reduces the average waiting time. 
In Fig. \ref{fig:waiting_time_newsgroups}, \method achieves 54s average waiting time for training Qwen2.5 on Newsgroups, compared with 72s, 123s, 157s, and 106s for GAS, PipeSFL, SplitFed, and MergeSFL, respectively.

\begin{table}[t]  
    \centering       
    \caption{Training Performance of Various Buffer Capacity}  
    \label{tab:key_param}     
    \renewcommand{\arraystretch}{1.4}
    \begin{adjustbox}{width=\linewidth}
        \begin{tabular}{c|c|c|c|c|c}
        \toprule 
        \diagbox{Metrics}{$\kappa$} &
    
        \makecell[c]{5\%} & 
        \makecell[c]{10\%} & 
        \makecell[c]{20\%} & 
        \makecell[c]{30\%} &
        \makecell[c]{50\%} \\
        \cline{1-6}
        Accuracy(\%) & 
        78.11 & 80.2 & 82.36 & 82.75 & 83.24 \\
        \cline{1-6}

        Time (s)& 
        9,151& 10,731& 12,553& 16,332& 21,016\\

        \bottomrule
    \end{tabular}
    \end{adjustbox}
\end{table}

\subsection{Effect of Non-IID Levels}
To demonstrate the effectiveness of the \method under non-IID data, we present the test accuracy of different approaches at varying non-IID levels in Fig. \ref{fig:noniid_level_acc}, in which the model accuracy of the five approaches on all the datasets decreases as the non-IID level increases.
However, the performance of \method is very close to MergeSFL and outperforms the other baselines on all datasets.
PipeSFL ignores statistical heterogeneity and exhibits the lowest accuracy.
To handle statistical heterogeneity, GAS employs an adaptive approach that samples and generates intermediate activations, while MergeSFL performs client selection via activation merging, partially mitigating non-IID effects on training.
Specifically, as illustrated in Fig. \ref{fig:noniid_level_acc_deit_cifar}, \method can achieve an improvement of test accuracy by about 2.06\%, 3.26\%, 1.76\% on CIFAR-100 with non-IID level of $p$=10, compared to the baselines (\ie, GAS, PipeSFL, SplitFed), and differs from MergeSFL by only 1.15\%.
Notably, by Fig. \ref{fig:noniid_level_acc_roberta_qnli}, with non-IID level of $p$=10 on QNLI, \method achieves an improvement of final accuracy by about 1.62\%, 2.44\%, 1.14\%, compared to the baselines (\ie, GAS, PipeSFL, SplitFed), and differs from MergeSFL by only 0.22\%.
Besides, as shown in Fig. \ref{fig:noniid_level_acc_qwen_newsgroup}, while transitioning from IID to non-IID level of $p$=10 on Newsgroups, \method and MergeSFL suffer from only 2.7\%, 1.5\% loss in accuracy, while the accuracy loss for GAS, PipeSFL, and SplitFed is 3.71\%, 4.15\%, and 3.84\%, respectively.
Collectively, these results demonstrate that \method maintains competitive model accuracy as statistical heterogeneity increases, confirming its robustness across different non-IID levels.

\subsection{Effect of Key Strategies}
There are two key strategies of \method, \ie, dynamic activation merging and consistency-aware model aggregation, being developed to enhance the performance of SFL.
Herein, we conduct several sets of experiments for training DeiT on CIFAR-100 with IID distribution ($p$=0) and non-IID distribution ($p$=10) to evaluate the effectiveness of the two critical strategies.
We adopt the \method without dynamic activation merging (\method w/o DAM) and \method without consistency-aware model aggregation (\method w/o CMA) as the baselines.
In \method w/o DAM, the PS directly updates the top model upon each activation arrival without activation selection or merging.
While in \method w/o CMA, the aggregation weight of each client follows the typical assignment method, \ie, the more iterations a client participates in, the greater the weight assigned.
By Fig. \ref{fig:key_strategies}, \method w/o CMA converges faster than \method w/o DAM on the IID and Non-IID datasets, but \method w/o DAM achieves a higher test accuracy than \method w/o CMA on the Non-IID dataset.
Specifically, \method w/o CMA reduces the total training time by about 56.1\% compared to \method w/o DAM, but \method w/o DAM improves the final test accuracy by about 1.27\% on the non-IID dataset, compared to \method w/o CMA.
Powered by dynamic activation merging and consistency-aware model aggregation, \method can speed up training by about 2.09 $\times$ compared to \method w/o DAM, and improve the final test accuracy by about 2.1\% compared to \method w/o CMA, which reflects the positive roles of the two strategies.

\subsection{Effect of Activation Buffer Capacity}
\label{sec:key_parameters}

We use the DeiT on CIFAR-100 to evaluate the performance of the \method under various buffer capacities $\Theta_0$ with a non-IID distribution ($p=10$).
The detailed results are presented in Table \ref{tab:key_param}.
For ease of evaluation, we define the buffer capacity ratio as the proportion of the total number of clients, \ie, $\kappa=\Theta_0/N$.
We observe two key trends.
First, as the ratio $\kappa$ ranges from 5\% to 20\%, the final accuracy improves notably.
This trend demonstrates that an insufficiently large buffer limits activation diversity, thereby affecting the quality of top model updates.
However, as $\kappa$ increases further from 20\% to 50\%, the accuracy plateaus, offering only a marginal gain of 0.88\%.
Conversely, the training time increases significantly with a larger buffer.
The time cost for increasing $\kappa$ from 5\% to 20\% is tolerable, but it increases rapidly by approximately 68.9\% as $\kappa$ grows from 20\% to 50\%, substantially weakening the efficiency benefit.
A larger buffer provides more candidate combinations but may also increase activation matching and waiting delays.
Considering this trade-off, $\kappa=20\%$ (\ie, $\Theta_0=N/5$) provides the most favorable balance, and we use this ratio in all other experiments.

\section{Related Work}\label{sec:relworks}

In the presence of pronounced heterogeneity in computation capability and communication bandwidth across clients, a common research addresses these challenges through structural adaptation of split models.
Several studies dynamically determine split points under heterogeneous resource constraints.
For example, SmartSplit \cite{tian2024breaking} analyzes each client's memory budget and runtime training capability to dynamically assign client-specific split layers, thereby improving training efficiency while maintaining model accuracy.
Similarly, ESFL \cite{ESFL2024Zhu} allocates bottom models of different depths to heterogeneous clients and jointly optimizes server-side computation resource allocation, enhancing training efficiency under system heterogeneity.
HFSL \cite{HFSL2024Wu} introduces auxiliary layers to accommodate heterogeneous model partitions and leverages self-distillation to improve training accuracy.
HeteroSFL \cite{HeteroSFL2025Chen} employs bottleneck layers of varying widths to accommodate heterogeneous communication capabilities and uses bidirectional knowledge sharing to mitigate non-IID data effects.
Similarly, HASFL \cite{lin2025hasfl} jointly optimizes adaptive batch-size scheduling and split-layer selection according to client capabilities, balancing training latency and convergence performance.
MergeSFL \cite{10597905} combines feature merging with batch-size regulation to address system and statistical heterogeneity under synchronous SFL.
Although these approaches improve training efficiency under heterogeneous resource constraints, most still rely on synchronous training, where the progress of each round is constrained by the slowest participating client. 
As computation and communication delays vary across clients, faster clients must wait at synchronization barriers, leaving their resources underutilized and limiting overall training efficiency. 

To alleviate this bottleneck, several studies further explore asynchronous execution, allowing the server to process available client activations without waiting for slower clients.
DT-ASFL \cite{stephanie2023digital} leverages digital twins to monitor device states and support asynchronous model updates.
PipeSFL \cite{gao2024pipesfl} adopts a priority-based scheduling strategy at the server and enables asynchronous execution within aggregation rounds to improve training efficiency.
From another perspective, RingSFL \cite{shen2024ringsfl} constructs a ring topology to distribute model training across heterogeneous clients and exploits overlapping layers to improve model updates.
These approaches primarily focus on alleviating straggler effects and improving execution efficiency, rather than explicitly addressing the optimization inconsistencies arising from asynchronous client participation.
GAS \cite{yang2025gas} further considers the biased top-model updates caused by unequal activation arrivals and generates supplementary activations to balance client influence.
However, its bottom-model aggregation does not explicitly account for the progress mismatch between the client and server model partitions.
In contrast, \method mitigates both the horizontal inconsistency caused by unequal client participation and the vertical inconsistency arising from progress mismatch between client-side bottom models and the server-side top model through dynamic activation merging and consistency-aware model aggregation.

\section{Conclusion}\label{sec:conclusion}
In this paper, we characterize horizontal inconsistency across clients and vertical inconsistency between clients and the server in asynchronous SFL.
Then we propose \method, an inconsistency-aware asynchronous SFL framework that incorporates a dynamic activation merging mechanism and a consistency-aware model aggregation mechanism to effectively suppress inconsistency propagation throughout asynchronous training.
The experimental results show that the \method can speed up model training by at least 1.61$\times$ to 3.74$\times$ and improve final model accuracy by up to 3.26\% in heterogeneous scenarios, compared to the baselines.


\bibliographystyle{IEEEtran}
\bibliography{bib/refs.bib}

\end{document}